\documentclass[12pt]{article}
\usepackage[margin=1in]{geometry}
\usepackage[small]{titlesec}

\usepackage{natbib,amsmath,amssymb,amsthm,bm,graphicx,caption,subcaption,paralist,booktabs}
\usepackage{quoting}
\quotingsetup{font=small,leftmargin=\parindent,rightmargin=0pt}
\usepackage[doublespacing]{setspace}
\AtBeginDocument{%
\abovedisplayskip=5pt plus 2pt minus 2pt
\belowdisplayskip=\abovedisplayskip
\abovedisplayshortskip=2pt plus 2pt minus 2pt
\belowdisplayshortskip=\belowdisplayskip}
\titlespacing*{\section}{0pt}{*2}{*1}
\titlespacing*{\subsection}{0pt}{*2}{*1}
\bibpunct{(}{)}{;}{a}{}{,}
\renewcommand\thefootnote{}

\allowdisplaybreaks[1]
\usepackage{nccmath}
\usepackage{graphicx}
\usepackage{enumitem}
\usepackage{array}
\usepackage{amsmath}
\usepackage{amsfonts}
\usepackage{amssymb}
\usepackage{color}
\usepackage{mathrsfs}
\usepackage[ruled, vlined]{algorithm2e}
\usepackage{tabularx}
\usepackage{booktabs}
\usepackage{blindtext}
\usepackage{multicol}
\usepackage{multirow}

\usepackage{siunitx,hyperref,lipsum}

\usepackage{hyperref}
\hypersetup{
colorlinks=true,
citecolor=blue
}

\usepackage{mathtools}
\usepackage{amsthm}
\theoremstyle{plain}
\newtheorem{theorem}{Theorem}[section]
\newtheorem{proposition}[theorem]{Proposition}
\newtheorem{lemma}[theorem]{Lemma}
\newtheorem{corollary}[theorem]{Corollary}
\theoremstyle{definition}
\newtheorem{definition}{Definition}
\newtheorem{condition}{Condition}
\theoremstyle{remark}

\newtheorem{example}{Example}

\def\msigma{\boldsymbol{\Sigma}}
\def\mpi{\boldsymbol{\Pi}}
\def\ie{i.e.}
\def\Rone{\uppercase\expandafter{\romannumeral1}}
\def\Rtwo{\uppercase\expandafter{\romannumeral2}}

\def\rmh{\mathbf{H}}

\def\matrixu{\bm{U}}
\def\vu{\bm{u}}
\def\mv{\bm{V}}  
\def\t{T}
\def\T{T}
\def\mA{\bm{A}}
\def\mD{\bm{D}}
\def\mI{\bm{I}}
\def\rvd{\mathbf{d}}

\def\vvartheta{\boldsymbol{\vartheta}}

\def\ma{\bm{A}}
\def\rcw{\textnormal{W}}

\def\md{\bm{D}}
\def\rmd{\mathbf{D}}
\def\vf{\bm{f}}
\def\re{\textnormal{e}}

\def\ry{\textnormal{y}}
\def\rvx{\mathbf{x}}
\def\rveps{\boldsymbol{\varepsilon}}
\def\va{\bm{a}}

\def\cz{\bm{\mathcal{Z}}}
\def\vx{\bm{x}}
\def\vz{\bm{z}}
\def\vv{\bm{v}}
\def\rvz{\mathbf{z}}

\def\vb{\bm{b}}
\def\vc{\bm{c}}
\def\vd{\bm{d}}
\def\rvg{\mathbf{g}}
\def\rca{\textnormal{A}}

\def\1{\mathbf{1}}
\def\cc{\mathcal{C}}
\def\cb{\mathcal{B}}
\def\vtheta{\boldsymbol{\theta}}
\def\vdelta{\boldsymbol{\delta}}
\def\mdelta{\boldsymbol{\Delta}}
\def\vpsi{\boldsymbol{\psi}}
\def\mtheta{\boldsymbol{\Theta}}
\def\vbeta{\boldsymbol{\beta}}
\def\sxi{\boldsymbol{\Xi}}
\def\mbeta{\mathbf{B}}

\def\mlambda{\boldsymbol{\Lambda}}
\def\valpha{\boldsymbol{\alpha}}
\def\malpha{\mathbf{A}} 
\def\0{\mathbf{0}}

\def\ball{\mathbb{B}}
\def\mi{\bm{I}}
\def\vs{\bm{s}}
\def\vy{\bm{y}}

\def\rmq{\mathbf{Q}}

\def\tr{\mathrm{tr}}

\def\ker{\mathrm{Ker}}

\def\gf{\textnormal{GF}}

\def\loc{\mathrm{loc}}
\def\his{\mathrm{his}}
\def\tr{\mathrm{tr}}

\def\bbeta{\boldsymbol{\beta}}
\def\valpha{\boldsymbol{\alpha}}
\def\mpsi{\boldsymbol{\Psi}}
\def\mupsilon{\boldsymbol{\Omega}}
\def\rmc{\mathbf{C}}
\def\vd{\bm{d}}
\def\rvg{\mathbf{g}}
\def\rca{\textnormal{A}} 
\def\rcb{\textnormal{B}}

\def\rone{\romannumeral1}
\def\rtwo{\romannumeral2}
\def\rthree{\romannumeral3}
\def\rfour{\romannumeral4}
\def\cB{\mathcal{B}}
\def\rfive{\romannumeral5}
\DeclareMathOperator*\argmin{\mathrm{argmin}}

\begin{document}
\begin{onehalfspace}

\title{Heterogeneous Federated Learning on a Graph}
\author{Huiyuan Wang\footnotemark[1], Xuyang Zhao\footnotemark[1], and Wei Lin\footnotemark[1]}
\date{}
\maketitle
\footnotetext[1]{School of Mathematical Sciences and Center for Statistical Science, Peking University, Beijing, China}

\let\thefootnote\relax\footnote{\emph{Address for correspondence}: Wei Lin, School of Mathematical Sciences and Center for Statistical Science, Peking University, Beijing 100871, China (E-mail: weilin@math.pku.edu.cn).}



\begin{abstract}
\noindent 

Federated learning, where algorithms are trained across multiple decentralized devices without sharing local data, is increasingly popular in distributed machine learning practice. Typically, a graph structure $G$ exists behind local devices for communication. In this work, we consider parameter estimation in federated learning with data distribution and communication heterogeneity, as well as limited computational capacity of local devices. We encode the distribution heterogeneity by parametrizing distributions on local devices with a set of distinct $p$-dimensional vectors. We then propose to jointly estimate parameters of all devices under the $M$-estimation framework with the fused Lasso regularization, encouraging an equal estimate of parameters on connected devices in $G$. We provide a general result for our estimator depending on $G$, which can be further calibrated to obtain convergence rates for various specific problem setups. Surprisingly, our estimator attains the optimal rate under certain graph fidelity condition on $G$, as if we could aggregate all samples sharing the same distribution. If the graph fidelity condition is not met, we propose an edge selection procedure via multiple testing to ensure the optimality. To ease the burden of local computation, a decentralized stochastic version of ADMM is provided, with convergence rate $O(T^{-1}\log T)$ where $T$ denotes the number of iterations. We highlight that, our algorithm transmits only parameters along edges of $G$ at each iteration, without requiring a central machine, which preserves privacy. To address communication heterogeneity, we further extend it to the case where devices are randomly inaccessible during the training process, with a similar algorithmic convergence guarantee. The computational and statistical efficiency of our method is evidenced by simulation experiments and the 2020 US presidential election data set.

\bigskip

\end{abstract}
\end{onehalfspace}

\newpage

\section{Introduction}
Intelligent devices, such as mobile phones, wearable devices, and autonomous cars, generate massive amounts of data every day. These data have a wide range of applications, for instance, the next word prediction \citep{konevcny2016federated}, scheduling the traffic to avoid the congestion \citep{accettura2013decentralized}, smoke detection \citep{khan2019energy}, and building health monitoring \citep{wu2020fedhome,scuro2018iot}.  In general, a graph structure exists among devices \citep{bello2014intelligent,atzori2010internet}. To accomplish certain tasks, different devices can communicate with each other along edges of such a graph, in addition to performing local computation. In fact,
concerns about private information leakage \citep{voigt2017eu}, coupled with the increasing computing power of devices, have made it a practice to store data as well as train algorithms locally and transmit only parameters. Besides, heterogeneity among devices naturally arises. For instance, the storage, computing and communication capabilities, and even the data distribution of each device may differ, and not all of the devices are accessible for a real-time training procedure \citep{li2020federated,srinidhi2019network}. These issues pose fundamental challenges to conventional distributed learning methodologies.

Distributed statistical methods commonly presume that a large dataset is randomly split into $K$ subsets stored on local devices that are connected to a central machine \citep{gao2021review}. Inspired by the idea of divide-and-conquer, small tasks are solved on local devices in parallel, and local results are aggregated on the central machine to produce the final result \citep{Qin2022selective}. Most of these approaches are statistically efficient and only require one round of communication, but are 
limited by the condition $K=o(\sqrt{N})$, where $N$ denotes the total sample size, e.g., \citet{JMLR:v14:zhang13b}, \citet{battey2018}, \citet{fan2019distributed}, among others. Despite that several methods are proposed to relax the condition on the number of local devices \citep{fan2021,Jordan1029communication}, this line of work essentially requires that tasks on each local device can be exactly solved, which can be problematic for devices with a limited computing power. This issue is even pronounced for stream data and/or for parameters required to be updated in real-time, e.g., automated cars \citep{Zhang2021real}.
To accommodate these situations, 
\cite{stich2019local} proposed local stochastic gradient descent (SGD) that runs SGD independently in parallel on
each device and averages the aggregated parameters on the central machine only once in a while. The iteration which parameter averaging takes place is referred to as `global synchronization', and the others are called `local update'. The number of global synchronization, i.e., communication cost, can be much smaller than the total number of iterations for algorithmic convergence. However, the aforementioned methods are all confined to homogeneous setups.

Federated learning \citep{konecny2016federated} is a machine learning technique tailored to the heterogeneity problem \citep{kairouz2019advances,zhang2020convergence}. One fundamental algorithm, federated averaging \citep{pmlr-v54-mcmahan17a}, generalizes local SGD in terms of communication heterogeneity by allowing certain devices to be randomly unavailable at each iteration. Notably, the graph whose edge represents the pathway that information can be transmitted is star-shaped for both local SGD and federated averaging, with each device connecting to a central machine. Due to limited bandwidth, it is increasingly impractical to send local parameters to a central machine for an ever-expanding network size \citep{Pavel2017mobile}.
To get rid of the central machine and also adapt to non-i.i.d data, \cite{koloskova2020unified} proposed a decentralized SGD by sampling random connected graphs at each iteration, where parameters of the neighbors in a simulated random graph, instead of all devices, are averaged at the synchronization step. This kind of algorithm is also known as the gossip algorithm \citep{boyd2006randomized}, which is communication efficient since only connected devices interact with each other at each iteration. 

Although effectively these inspiring algorithms can be computed, device heterogeneity therein lacks rigorous statistical formulation. For example, both federated averaging and decentralized SGD are designed for the situation where parameters of devices are the same.  However, parameter equality indeed implies distribution identity if parameters are identifiable; see \citet[Chap.5]{van2000asymptotic} for details. Relating heterogeneous data distribution to parameter estimation, we can measure the goodness of algorithms in dealing with heterogeneity in terms of a metric between estimators and the true parameter \citep{cai2021shir}. This line of work commonly assumes that effects of covariates on outcomes can be decomposed into a common effect shared by all local devices and device-specific effects that explain heterogeneity, e.g., \citet{Zhao2016partial} and \citet{duan2021heterogeneity}. Under this parametrization of heterogeneity, covariates giving rise to device-specific effects require to be known, which demands a strong prior knowledge. Moreover, algorithms therein cannot be executed in real-time, which makes them mostly suitable to multi-center research \citep{sidransky2009multicenter} instead of federated learning.

In this work, we encode data heterogeneity with an unknown graph $G_0=(V,E_0)$, where $V$ consists of all devices as nodes and $G_0$ is defined as a collection of multiple disjoint cliques. Each clique essentially defines a cluster of devices sharing the same distribution identified by some unknown $p$-dimensional parameters. We further assume that a graph $G=(V,E)$ is given in priori with possibly $E\neq E_0$, whose edges not only represent communication pathways
but also reveal certain similarities between connected devices. In social networks \citep{scott1988social,wasserman1994social}, the property that linked nodes act similarly is known as network cohesion \citep{li2019prediction}, a phenomenon observed in numerous social behavior studies \citep{christakis2007spread,fujimoto2012social}. 

We promote the estimators on connected devices to be equal by the following $M$-estimation penalized by the fused Lasso:
\begin{align}\label{eq:penalized_m}
\min_{\vtheta_u\in \sxi, u\in V} \frac{1}{|V|}\sum_{u\in V}\frac{1}{n_u}\sum_{k=1}^{n_u}m_u(\rvz_k^{(u)};\vtheta_u)+ \lambda\sum_{(i,j)\in E}\phi(\vtheta_i-\vtheta_j), 
\end{align}
where $m_{u}(\cdot;\vtheta), u\in V$, denote some known functions, $\sxi\subset \mathbb{R}^p$ denotes the parameter space which is a compact, $\{\rvz_k^{(u)}\}_{k=1}^{n_u}$ denotes the dataset observed on device $u$, and $\phi(\cdot)$ is some norm defined on $\mathbb{R}^p$. Notably, a wide range of statistical models, including (generalized) linear models, Huber regression, and maximum likelihood estimation, are included by choosing corresponding $m_{u}(\cdot;\vtheta), u\in V$. 
We propose Fed-ADMM, a decentralized stochastic version of alternating direction method of multipliers (ADMM) algorithm, to solve \eqref{eq:penalized_m}.  Quite different from federated averaging, our algorithm does not require a central machine and parameters are transmitted device-to-device. Our algorithm does not require a high computational capacity of local devices either; at each iteration, only a mini-batch of samples is processed on each device. We prove that it achieves $O(T^{-1}\log T)$ convergence rate to the global minimizer of \eqref{eq:penalized_m} where $T$ denotes the total number of iterations. We further show that our algorithm attains the same convergence rate under malicious random block of devices in the real-time optimization process. 
Moreover, we provide a deterministic theorem on the consistency of the global minimizer of \eqref{eq:penalized_m}.  Interestingly, the estimation error can be decomposed into two components, the averaged variance of $M$-estimation without regularization from the data fidelity term in \eqref{eq:penalized_m}, and the bias term introduced by wrongly shrinking $\vtheta_i -\vtheta_j, (i,j)\in E\setminus E_0$ towards zero due to the fused regularization. In certain specific cases, we further obtain probabilistic results for the estimation error. Our result shows that, if the given graph $G$ is close enough to $G_0$, the global minimizer of \eqref{eq:penalized_m} performs optimally as if we could aggregate the whole data and know which devices share the same distribution. For the case where $G$ provides too much misleading information, we propose an adaptive edge selection procedure through multiple testing to ensure the optimality.

\subsection{Related Work}
Our work is closely related to personalized federated learning which, vaguely speaking, aims to personalize global models to perform better for individual devices. There is a growing body of literature that focuses on personalized federated learning. \cite{mansour2020three} proposed to cluster similar devices first and then applied federated averaging to each cluster. Taking a perspective of transfer learning, \cite{wang2019federated}
proposed to train a global model first, and then some or all of parameters of the global model are fine-tuned on local devices. \cite{jiang2019improving} borrowed some ideas from  model agnostic meta learning to deal with personalized federated learning problems. 
\cite{smith2017federated} proposed MOCHA to produce personalized solutions in light of multi-task learning. For a more comprehensive review, see \cite{kulkarni2020survey}. These methods, however, either serve as pure algorithmic solutions that lack rigorous statistical analyses, or require a strong computing power of local devices.

Another line of work related to this article is network/fused Lasso \citep{hallac2015network,tibshirani2005sparsity,rudin1992nonlinear}. This line of work presumes that parameters are sparse over a given graph/network. \cite{hutter2016optimal} derived a sharp convergence rate for Gaussian mean estimation with total variation regularization. \cite{hallac2015network} used distributed ADMM to solve optimization problems of the same form as \eqref{eq:penalized_m}. Our work generalize \cite{hutter2016optimal} to general $M$-estimation and \cite{hallac2015network} to a stochastic federated setting.  
\cite{richards2021distributed} considered the case where each node is associated with a sparse
linear model, and two nodes are linked if the difference of their solutions is also sparse. They required that the underlying graph is a tree and sparsity
of the differences across nodes is smaller than the sparsity at the root. We, however, consider arbitrary graphs.

\subsection{Organization of This Paper}
The rest of the paper is organized as follows. Some useful notation, problem setup and related assumptions are discussed in Section \ref{sec:pre}. Section \ref{sec:method} mainly contains details of our method, including statistical guarantees of our estimators and an edge selection procedure through multiple testing. In Section \ref{sec:optim}, we introduce the Fed-ADMM together with its extension and show their algorithmic consistency. Section \ref{sec:simulation} consists of simulations, and a real-world data analysis is included in Section \ref{sec:real}.
\section{Preliminaries}\label{sec:pre}
We first introduce some notation used in this article.
\subsection{Notation}
For any set $S$, we denote by $|S|$ its cardinality. We denote by $G=(V,E)$ a graph with node set $V=\{1,\ldots,|V|\}$ and edge set $E$. For any $e=(i,j)\in E$, we let $e^+=\max\{i,j\}$ and $e^-=\min\{i,j\}$. The signed incidence matrix with respect to $E$ is denoted by $\md\in\{-1,0,1\}^{|E|\times |V|}$ whose $(e,i)$-th entry is $D_{e i}=\1\{i = e^+\} - \1\{i = e^-\} $ for $e\in E$ and $i\in V$, where $\1\{\cdot\}$ denotes the indicator function. We let $\md^{\dag}$ be the Moore--Penrose inverse of $\md$. We denote by $\cc_1,\ldots,\cc_{K}$ the connected components of a graph $G=(V,E)$, where $K$ is the number of $\cc_i$s in $G$. We sometimes emphasis the dependence of $K$ on $G=(V,E)$ by writing $K(E)$. For any matrix $\ma=(\va_1,\ldots,\va_m)^\t\in\mathbb{R}^{m\times p}$, define the $\ell_1/\phi$-norm of $\ma$ as $R(\ma)=\sum_{j=1}^m\phi(\va_j)$, where $\phi:\mathbb{R}^p\mapsto [0,\infty)$ denotes a norm defined on $\mathbb{R}^p$. Sometimes, we also write $\ma=(\va_k: k\in S)$ a $|S|\times p$ as a matrix whose $j$-th row is $\va_j^\t$. We denote by $\|\va\|_2$ the $\ell_2$-norm of $\va$ defined on $\mathbb{R}^p$. Let $\ball(\va;r)=\{\vx: \|\vx - \va\|_2\le r\}$ be the ball in $\mathbb{R}^p$ with center being $\va$ and radius being $r$. For some symmetric matrix $\ma$, we denote by $\lambda_{\max}(\ma)$ the maximal eigenvalue of $\ma$ and by $\lambda_{\min}(\ma)$ its minimal eigenvalue. If $\ma$ is positive semi-definite, we denote by $\lambda_{\min}^{+}(\ma)$ its smallest nonzero eigenvalue.

\subsection{Problem Setup and Identifiability of Heterogeneity}

We consider the heterogeneity of devices in federated learning.  A device can be different from other devices in a variety of ways, including data distribution, computing power, and availability patterns during communication. These types of spatial and temporal heterogeneity can be described by graphs whose nodes represent devices. Consider probability measures $P(\vtheta_u^*)$, $u\in V$, with $\{\vtheta_u^*\colon u\in V\}\subset \ball(\0_p;r_0)\subset\mathbb{R}^p$ for some constant $r_0\in (0,\infty)$. Naturally, we denote by $\sxi=\ball(\0_p;r_0)$ the parameter space. We 
assume that, for each $u\in V$, $\rvz_k^{(u)}\sim P(\vtheta_u^*), 1\le k\le n_u$ independently on device $u$. We denote by $\cz_u\subset\mathbb{R}^{q_u}$ the support set of $P(\vtheta_u^*)$ such that $\{\rvz_k^{(u)}\}_{k=1}^{n_u}\subset \cz_u$. We assume that $\cz_u$ is compact such that $\max_{u\in V}\sup_{\vz_1,\vz_2\in\cz_u}\|\vz_1^{(u)}-\vz_2^{(u)}\|_2=r_z\in (0,\infty)$. Let $n=\min_{u\in V}n_u$.
\begin{definition}[Characteristic graph]\label{def:data_dist}
A graph $G_0=(V,E_0)$ is called a characteristic graph of a set of probability distributions $\{P(\vtheta_u^*); u\in V\}$ 
if $\vtheta_u^*=\vtheta_v^*$ is equivalent to $(u,v)\in E_0$. 
\end{definition}
A characteristic graph has a natural decomposition $G_0=\cup_{i=1}^{K(E_0)}\cc_i^*$ where $\cc_i^*$s $\big(1\le i\le K(E_0)\big)$ denote disjoint cliques. This definition is, in fact, an intermediate model between local and global models, able to adaptively characterize the degree of heterogeneity. If $K_*=1$, which means $G_0$ is complete and thus parameters of local devices equal to each other, our model is reduced to the global model considered in \cite{stich2019local} and \cite{pmlr-v54-mcmahan17a}. If $K_*=|V|$, which means all parameters are distinct from each other, we arrive at the local model and recover the setting of \cite{smith2017federated}. In practice, $E_0$ and the number of clusters $K(E_0)$ are typically unknown. 

Under a general $M$-estimation framework, $\vtheta_u^*$ can be identified by
\begin{align}\label{eq:M_est}
\vtheta_u^*=\argmin_{\vtheta\in\sxi}M_u(\vtheta)\equiv E[m_u(\rvz;\vtheta)], \quad u\in V.
\end{align}
Notably, the model \eqref{eq:M_est} consists of a wide variety of statistical models, including linear regression, generalized linear regression, maximum likelihood estimation, and so on. Moreover, methods employed for parameter estimation
are allowed to vary across devices. For example, at device $u$ we are interested in the population mean $\vtheta_u^*=E(\rvz^{(u)})$ which can be estimated by maximizing the likelihood of observed samples, while at device $v$ we are given a classification task and $\vtheta_v^*$ is the coefficient of classification hyperplane which can be estimated by logistic regression. 

Regularity conditions are required to guarantee that \eqref{eq:M_est} is well-posed. Denote by $\vpsi_u(\cdot;\vtheta)$ the derivative of $m_u(\cdot;\vtheta)$ with respect to $\vtheta$, and let $\msigma_u(\vtheta)=\mathrm{cov}(\vpsi_u(\rvz;\vtheta))$ be the covariance matrix of $\vpsi_u(\rvz;\vtheta)$ with $\rvz\sim P(\vtheta_u^*)$. Also, denote by $\rmh_u(\vtheta)$ the Hessian matrix of $M_u(\vtheta)$, and by $\widehat{\rmh}_u(\vtheta)$ its empirical counterpart, \ie, $\widehat{\rmh}_u(\vtheta) = n_u^{-1}\sum_{k=1}^{n_u}\partial \vpsi(\rvz_k^{(u)};\vtheta)/\partial \vtheta$.
\begin{condition}[Identifiability]\label{con:identifiability}
For each $u\in V$, $m_u(\rvz,\vtheta)$ is convex and twice differentiable with respect to $\vtheta$ within $\sxi$. Moreover, $\rmh_u(\vtheta)$ is Lipschitz continuous under the operator norm at $\vtheta=\vtheta_u^*$, i.e., for any $\vtheta\in \sxi$,
\[
\|\rmh_u(\vtheta) - \rmh_u(\vtheta_u^*)\|_2\le L \|\vtheta - \vtheta_u^*\|_2, 
\]
and the eigenvalues of $\rmh_u(\vtheta_u^*)$ are bounded, i.e.,
\[
\underline{\lambda} \le \min_{u\in V}\lambda_{\min}(\rmh_u(\vtheta_u^*)) \le \max_{u\in V}\lambda_{\max}(\rmh_u(\vtheta_u^*))\le \overline{\lambda} ,
\]
where $L, \underline{\lambda}$ and $\overline{\lambda}$ are some constants independent of $u\in V$ such that $\sxi \supset \cup_{u\in V}\ball(\vtheta_u^*;\underline{\lambda}/(2L))$.
\end{condition}
Condition \ref{con:identifiability} implies that for each $u\in V$, $\vtheta_u^*$ is locally identifiable; that is, it is the unique minimizer of $M_u(\vtheta)$ at the vicinity of $\vtheta_u^*$; see Lemma \ref{lem:strong_c}. In light of Definition \ref{def:data_dist}, distribution heterogeneity among devices is uniquely characterized in terms of $\mtheta^*=(\vtheta_u^*:u\in V)$. In order to  estimate $\mtheta^*$ from samples, we also need to impose certain conditions on $P(\vtheta_u^*), u\in V$. 
\begin{condition}[Design and noise distribution]\label{con:distribution}
\begin{itemize}
\item[(\rone)]\textbf{Sub-Gaussian noises.} For each $u\in V$, the random vector $\vpsi_u(\rvz_k^{(u)};\vtheta^*_u)$ is sub-Gaussian with parameter $\sigma^2\in(0,\infty)$, \ie,
\[
E [\exp \{( \va^\t\vpsi_u(\rvz_k^{(u)};\vtheta^*_u) )^2 / \sigma^2\}] \le 2,\quad  k = 1, \dots, n_u,
\]
for any $ \va \in \mathbb{R}^p$ with $\| \va \|_2 = 1$.
\item[(\rtwo)]\textbf{Fixed design.} It holds that 
\begin{align*}
\kappa^{-1} &\le \min_{u\in V}\inf_{\vtheta\in \ball(\vtheta_u^*;r_1)}\lambda_{\min}(\widehat{\rmh}_u(\vtheta))\le \max_{u\in V}\sup_{\vtheta\in \ball(\vtheta_u^*;r_1)}\lambda_{\max}(\widehat{\rmh}_u(\vtheta))\le \kappa
\end{align*}
for some constant $\kappa\ge 1$ and $r_1>0$ with $\cup_{u\in V}\ball(\vtheta_u^*;r_1)\subset \sxi$.
\item[(\rthree)]\textbf{Random design.} For any $\vtheta\in\sxi$,
the 
derivative of $\vpsi_u(\cdot;\vtheta)$ with respect to $\vtheta$ can be decomposed as
\[
\frac{\partial \vpsi_u(\cdot;\vtheta)}{\partial \vtheta} = h_u(\cdot;\vtheta) \vf_u(\cdot) \vf_u(\cdot)^\t,
\]
for a vector-valued function $\vf_u(\cdot): \mathbb{R}^{q_u}\to \mathbb{R}^p$ and a scalar-valued function $h_u(\cdot;\vtheta)$ parametrized by $\vtheta$. Moreover,
\[
\underline{h}\le \min_{u\in V}\inf_{\vz\in\cz_u}\inf_{\vtheta\in \sxi}|h_u(\vz;\vtheta)|\le \max_{u\in V}\sup_{\vz\in\cz_u}\sup_{\vtheta\in \sxi} |h_u(\vz;\vtheta)|\le \overline{h}
\]
holds for some constants $\overline{h}\ge \underline{h}>0$, and $\vf_u(\rvz) \vf_u(\rvz)^\t$ with $\|E(\vf_u(\rvz) \vf_u(\rvz)^\t)\|_2 \le \sigma_{1,z}^2$ is a sub-exponential random matrix with parameters $(\mv,\alpha)$ for $\rvz\sim P(\vtheta_u^*)$, i.e.,
\[
E\left\{\exp\left[\lambda\left( \vf_u(\rvz) \vf_u(\rvz)^\t-E\left(\vf_u(\rvz) \vf_u(\rvz)^\t\right)\right)\right]\right\}\preceq \exp\left(\frac{\lambda^2\mv}{2}\right), 
\]
for $ |\lambda|<\alpha^{-1}$, 
where the matrix-valued parameter $\mv\in\mathbb{R}^{p\times p}$ satisfies  $ \|\mv\|_2\le \sigma_{2,z}^2$, and $ 
\sigma_{1,z}, \sigma_{2,z}$ are some positive constants.
\end{itemize}
\end{condition}
For supervised learning, part (\rone) of Condition \ref{con:distribution} imposes a sub-Gaussian tail on the distribution of noises for technical convenience, which can be relaxed to distributions of high-order moments. Part (\rtwo) of Condition \ref{con:distribution} requires that the empirical Hessian matrices have a bounded condition number, which is standard in the federated learning literature. We show in Lemma \ref{lem:conc_emp_hes} that it holds with high probability under part (\rthree) of Condition \ref{con:distribution}. 
Part (\rthree) requires that the Hessian matrix of $m_u(\cdot;\vtheta)$ can be factorized into a product, such that one factor is a sub-exponential random matrix which does not depend on $\vtheta$, and the other factor is parameter-dependent and uniformly bounded over the parameter space. We show that various popular statistical models satisfy this factorization in Section \ref{sec:method}. Notably, either part (\rtwo) or (\rthree) is enough to guarantee that $\mtheta^*$ can be recovered from finite samples.

Once \eqref{eq:M_est} being well-posed, a naive estimator of $\vtheta_u^*$ is 
\begin{align}\label{eq:naive_est}
\widehat{\vtheta}_u^{\loc} = \argmin_{\vtheta\in\sxi}\widehat{M}_u\equiv \frac{1}{n_u}\sum_{k=1}^{n_u}m_u(\rvz_k^{(u)};\vtheta).
\end{align}
Here, the dimensionality, $p$, is allowed to increase with $n$, but in a low dimensional manner, i.e., $p= o(n)$. Under Condition \ref{con:identifiability} and \ref{con:distribution}, $\widehat{\vtheta}_u^{\loc} - \vtheta_u^*$ is asymptotically normal; see Proposition \ref{prop:asymptotic_normality}. However, since $\vtheta_u^*=\vtheta_v^*$ for any $(u,v)\in E_0$, estimating each $\vtheta_u^*$ separately gives rise to an efficiency loss compared to the ideal estimator obtained by aggregating all data points with the same distribution \citep{dobriban2021distributed}. To avoid the efficiency loss, the latent structure $G_0$ among local devices requires to be considered. 

We shall introduce our estimator in the next section, and compare its performance with the naive estimator $\widehat{\mtheta}^{\loc}=(\widehat{\vtheta}^{\loc}_u : u\in V)$.

\section{Methodology}\label{sec:method}
When the characteristic graph $G_0=(V,E_0)$ is known, our heterogeneous problem is reduced to $K(E_0)$ independent homogeneous problems. Unfortunately, $G_0$ is unknown. In this article, we consider the case where $G=(V,E)$, a surrogate graph of $G_0$, is given a priori, and establish a detailed statistical analysis of the effect introduced by wrong edges $(E\setminus E_0)\cup (E_0\setminus E)$ when $G$ deviates from $G_0$. In practice, $E$ can be acquired from prior knowledge of $E_0$ or communication pathways among devices.

When $E$ incorporates prior information of $E_0$, it is beneficial to introduce a penalty that encourages an equal estimate on connected devices in $G$. By doing so, we borrow information from devices sharing the same data distribution. 
In more specific, we propose the following penalized $M$-estimation:
\begin{equation}\label{eq:ob}
\widehat{\mtheta}=(\widehat{\vtheta}_u: u\in V)=\argmin_{\mtheta}F(\mtheta)= \frac{1}{|V|}\sum_{u\in V} \widehat{M}_u(\vtheta_u)+ \lambda R(\md\mtheta),
\end{equation}
where $\mtheta=(\vtheta_u: u\in V)\in\mathbb{R}^{|V|\times p}$, 
$\md\mtheta=(\vtheta_{e^+} - \vtheta_{e^-}: e\in E)\in\mathbb{R}^{|E|\times p}$, $\lambda$ is a tuning parameter, and $R(\md\mtheta)=\sum_{e\in E}\phi(\vtheta_{e^+} - \vtheta_{e^-})$ where $\phi(\cdot)$ is a norm defined on $\mathbb{R}^p$. 
We first show that a wide range of statistical models are included in \eqref{eq:ob} through several examples.
\begin{example}[Mean estimation]
For some $u\in V$, samples are drawn independently from following model
\begin{align}
\rvz^{(u)}_k = \vtheta_u^* + \rveps^{(u)}_k,\quad 1\le k\le n_u,\ u\in V, \label{eq:mean_Example}
\end{align}
where $\vtheta_u^*\in\mathbb{R}^p$ is the unknown target mean vector on device $u$, and $\rveps^{(u)}_k$s are independent Gaussian noises with variance matrix $\sigma^2\mi_p$. Choose the negative log-likelihood as $m_u(\rvz_k^{(u)}; \vtheta)=\|\rvz^{(u)}_k-\vtheta\|_2^2/(2\sigma^2)$, and thus $\vpsi_u(\rvz_k^{(u)}; \vtheta)=\sigma^{-2}(\vtheta-\rvz^{(u)}_k)$. Since noises are normal distributed with bounded variances, Condition \ref{con:identifiability} and \ref{con:distribution} hold. 
For $n_u= 1 (\forall u\in V)$, \eqref{eq:mean_Example} is essentially nonparametric \citep{tibshirani2005sparsity,Padilla2020Adaptive}, and $\vtheta_u^*$ can be viewed as the value that some function $f_0(\cdot)$ takes on device $u$, in which case a device is a data point. In particular, if further $p=1$ and $G$ is a grid graph, \eqref{eq:mean_Example} is reduced to total variation de-noising \citep{hutter2016optimal}.

\end{example}

\begin{example}[Linear regression]
Suppose that for some $u\in V$, $\rvz_k^{(u)}=(\rvx_k^{(u)}, \ry_k^{(u)})\in\mathbb{R}^p\times \mathbb{R}$ is generated independently by
\[
\ry_k^{(u)} = (\vtheta_u^*)^\t \rvx_k^{(u)} + \varepsilon_k^{(u)}, \quad 1\le k\le n_u,
\]
where $\varepsilon_k^{(u)}$s are independent sub-Gaussian noises with parameter $\sigma^2$, and $\{\rvx_k^{(u)}\}_{k=1}^{n_u}$ is fixed with non-singular covariance matrix $\msigma_u$. Choose the squared error loss as $m_u(\rvz_k^{(u)};\vtheta) = (\ry_k^{(u)} -  \vtheta^\t \rvx_k^{(u)} )^2 /2$, then $\vpsi_u(\rvz_k^{(u)};\vtheta) = -(\ry^{(u)}_k- \vtheta^\t \rvx_k^{(u)}) \rvx_k^{(u)}$ and $\msigma_u(\vtheta_u^*) = \sigma^2\msigma_u$. It is straightforward to see that Condition \ref{con:identifiability} and part (\rone) and part (\rtwo) of Condition \ref{con:distribution} are satisfied.
\end{example}

\begin{example}[Logistic regression]
Suppose that for some $u\in V$, $\rvz_k^{(u)}=(\rvx_k^{(u)}, \ry_k^{(u)})\in\mathbb{R}^p\times \{0,1\}$ is generated independently by
\[
P(\ry_k^{(u)} \mid \rvx_k^{(u)})=\exp \left\{\ry_k^{(u)} (\vtheta_i^*)^\t \rvx_k^{(u)} - \zeta((\vtheta_u^*)^\t \rvx_k^{(u)} )\right\}, \quad 1\le k\le n_u,
\]
where $\zeta(t) = \log(1+\exp(t))$. As a standard assumption, $\{\rvx_k^{(u)}\}_{k=1}^{n_u}$ are i.i.d samples from a distribution supported on the unit sphere $\ball(\0_p;1)$ with $c\mi_p\succeq \mathrm{var}(\rvx_k^{(u)})\succeq c^{-1}\mi_p$ for some constant $c>1$. Choose the negative log-likelihood as $m_u(\rvz_k^{(u)};\vtheta) = \zeta(\vtheta^\t \rvx_k^{(u)})- \ry_k^{(u)} \vtheta^\t \rvx_k^{(u)}$. Then $\vpsi_u(\rvz_k^{(u)};\vtheta) = \zeta^\prime(\vtheta^\t \rvx_k^{(u)})\rvx_k^{(u)} -\ry_k^{(u)} \rvx_k^{(u)}$ and $(\partial \vpsi_u/\partial \vtheta)(\rvz_k^{(u)};\vtheta) = \zeta^{\prime\prime}(\vtheta^\t \rvx_k^{(u)})\rvx_k^{(u)}(\rvx_k^{(u)})^\t$. One can show that Condition \ref{con:identifiability} and \ref{con:distribution} hold since $\sup_{\rvx\in \ball(\0_p;1), \vtheta\in \sxi}|\vtheta^\t \rvx|\le \sup_{\vtheta\in\sxi}\|\vtheta\|_2\le r_0$.
\end{example}
Clearly, the regularization term $\lambda R(\md\mtheta)$ promotes $\vtheta_{e^+} - \vtheta_{e^-}$ to be zero for any $e\in E$. If $\lambda = 0$ (\ie, the extra information of $G$ does not play any role), \eqref{eq:ob} is reduced to the local estimator defined in \eqref{eq:naive_est}, since the data fidelity term $|V|^{-1}\sum_{u\in V} \widehat{M}_u(\vtheta_u)$ is separable with respect to $\vtheta_u$. On the other hand, if $\lambda\to\infty$, \eqref{eq:ob} is reduced to the global estimator that presumes that $\vtheta_u^*\equiv \vtheta^* (\forall u\in V)$, ignoring distribution heterogeneity among devices. In the following, we shall provide a general statistical convergence guarantee of $\widehat{\mtheta}$, which can be used to determine the rate of $\lambda$, and to adaptively adjust for distribution heterogeneity. 

\subsection{Statistical Guarantees}\label{sec:sg}
Let $\mdelta^* = \md\mtheta^*=(\vdelta_e^*:e\in E)$, then  $\vdelta_e^*=\vtheta_{e^+}^* - \vtheta_{e^-}^*\neq 0$ if and only if $e\in E\setminus E_0$. Inspired by the group lasso \citep{negahban2012unified}, we can recover the nonzero rows of $\mdelta^*$ under the help of $\lambda \sum_{e\in E}\phi(\vdelta_e^*)$. In high-dimensional settings, some regularity condition, e.g., compatibility condition \citep{buhlmann2011statistics} and restricted eigenvalue condition \citep{bickel2009simultaneous}, of the design matrix is required for parameter estimation. In our case, the `design matrix' with respect to $\mdelta^*$ is not explicitly given, but closely related to the incidence matrix $\md$. We use a similar notion, compatibility factor \citep{hutter2016optimal}, to characterize the regularity of `design matrix' for estimating $\mdelta^*$.
\begin{definition}[Compatibility factor]\label{def:cf}
The compatibility factor of $\md$ for a set $T\subset E$ with respect to $R(\cdot)$ is defined as
\[\kappa_{\emptyset}(\md)\equiv 1,\quad \kappa_T(\md)\equiv \inf_{\mtheta \in \mathbb{R}^{|V|\times p}}\frac{\sqrt{|T|} \|\mtheta\|_F}{R[(\md\mtheta)_T]}\quad \mathrm{for}\ T\neq\emptyset, \]
where $(\md\mtheta)_{T,:} = (\vtheta_{e^+} - \vtheta_{e^-}: e\in T)\in\mathbb{R}^{|T|\times p}$ denotes the sub-matrix of $\md\mtheta$ with rows indexed by $T$. 
\end{definition}
The compatibility factor $\kappa_T(\md)$ is similar in spirit to the compatibility condition in the Lasso literature. To see this, for $p=1$ and $\md^{\dag}\md= \mi_{|V|}$, it holds that $\mtheta\in \mathbb{R}^{|V|}, \mdelta=\md\mtheta\in \mathbb{R}^{|E|}$, $R[(\md\mtheta)_{T}]=\|\mdelta_T\|_1$, and $\|\mtheta\|_F^2 = \mdelta^\t (\md^{\dag})^\t \md^{\dag} \mdelta$. Thus $(\md^{\dag})^\t \md^{\dag}$, as the `design matrix' for estimating $\mdelta^*$, satisfies the compatibility condition with constant $\kappa_T^2(\md)$. Also, the compatibility factor helps to identify nonzero rows of $\mdelta^*$ through the following condition.

\begin{condition}\label{con:comp_graph}
The compatibility factor $\kappa_S(\md)\ge \kappa_0$ for $S=E\setminus E_0$, where $\kappa_0>0$ denotes some universal constant. 
\end{condition}

In fact, $\kappa_T(\md)$ can also be viewed as a measure of the degree centrality \citep{Sharma2013} of $G=(V,E)$. Following \citet[Lemma 3]{hutter2016optimal}, we can show that $\kappa_T(\md)\ge 1/(2\min\{\sqrt{d},\sqrt{|T|}\})$, where $d$ denotes the maximum degree of $G$. For graphs with a bounded maximal degree, Condition \ref{con:comp_graph} is met. 


\subsubsection{Deterministic Results}

We have the following deterministic statement about the global minimizer of
the convex program \eqref{eq:ob}. For convenience, let $\widehat{\mpsi}(\mtheta^*) = (\nabla_{\vtheta_u} \widehat{M}_u(\vtheta_u^*): u\in V)\in\mathbb{R}^{|V|\times p}$, and let $R^*(\cdot)$ be the dual norm of $R(\cdot)$ defined in Lemma \ref{lem:Rnorm}.
\begin{theorem}\label{thm:main}
Under the Condition \ref{con:identifiability}, part (\rtwo) of Condition \ref{con:distribution}, and Condition \ref{con:comp_graph}, let $\widehat{\mtheta}$ be the global minimizer of \eqref{eq:ob}, $\rho = |V|^{-1/2}\|\mpi_{\ker(\md)}\widehat{\mpsi}(\mtheta^*)\|_F$ and $\lambda = |V|^{-1/2} R^*\{(\md^\dag)^\t\widehat{\mpsi}(\mtheta^*)\}$, where $\mpi_{\ker(\md)}$ denotes the projection matrix that mapping vectors in $\mathbb{R}^{|V|}$ to the kernel space of $\md$. If  $\{\widehat{\vtheta}_u\}_{u\in V}\subset \sxi$, we have that
\begin{align}\label{thm_main}
\frac{1}{|V|}\|\widehat{\mtheta}-\mtheta^*\|_F^2&\le 2\kappa^2 \left(\rho^2 + \frac{4|S|}{\kappa_0}\lambda^2\right).
\end{align}   
\end{theorem} 
We remark that the strong convexity condition, \ie, part (\rtwo) of Condition \ref{con:distribution}, is necessary.  Owing to the rank deficiency of incidence matrix $\md$, there exists a subspace of $\mtheta$ that cannot be penalized. For example, even if we would have $\widehat{\vtheta}_i=\widehat{\vtheta}_j$ for $(i,j)\in E\cap E_0$, there could exist a common shift in both $\widehat{\vtheta}_i$ and $\widehat{\vtheta}_j$, such that $\widehat{\vtheta}_i-\vtheta_i^*=\widehat{\vtheta}_j-\vtheta_j^*=c\neq 0$. The strong convexity of empirical Hessian matrices 
ensures the uniqueness of each $\widehat{\vtheta}_i$ and controls the shift effect. 

Our deterministic result builds on the condition $\lambda=|V|^{-1/2} R^*\{(\md^\dag)^\t\widehat{\mpsi}(\mtheta^*)\}$, which can help us to determine the rate of $\lambda$ for some specific choice of $R(\cdot), \{m_u(\cdot;\vtheta)\}_{u\in V}$, and $G=(V,E)$.
We interpret $|S|\lambda^2/\kappa_0$ as the mean squared error term incurred by estimating $\mdelta^*= \md\mtheta^*$. Owing to the $\ell_1$ part of $R(\cdot)$-norm, we can estimate such a row-wise sparse matrix with a Lasso-type convergence rate.  We highlight that there exists an interesting trade-off in $|S|\lambda^2/\kappa_0$. Specifically, choosing $\phi(\cdot)=\|\cdot\|_1$, the constants in \eqref{eq:lambda_highp} consist of $\gamma_G$ and $\kappa_0$. For a connected graph $G$, the quantity $\gamma_G$ denotes the algebraic connectivity of $G$. A larger $\gamma_G$ represents that $G$ is more connected, which potentially gives rise to a larger $|S|$ and smaller $\kappa_0$. 

We interpret $\rho^2$ as the \textit{averaged intra-group variances} of devices for estimating $\mtheta^*$ with respect to the given graph $G$. The notion group denotes each connected component of $G$, since $\ker(\md)=\mathrm{span}\{\1\{\cc_1\},\ldots,\1\{\cc_{K(E)}\}\}$ \citep[see][Theorem 8.3.1]{Godsil2001}, where $\1\{\cc_i\}\in\mathbb{R}^{|V|}$ denotes the indicator vector of $\cc_i$ whose component in $\cc_i$ is $1$ and otherwise $0$. Notice that devices can only communicate within each connected component of $G$. Smaller $K(E)$, implied by a stronger communication ability, gives rise to a smaller $\rho^2$. However, stronger communication ability comes with a price. More connected $G$ is, the size of $S = E\setminus E_0$ can be potentially larger, which reveals an interesting guidance for the choice of $G=(V,E)$; that is, we need to balance connectivity and the number of false edges. We emphasize that $\rho^2$ is unavoidable to identify $\mtheta^*$, even if $E=E_0$, in which case $E\rho^2= O\big(\sigma^2 p K(E_0) /(n |V|)\big)$, demonstrating a parametric rate, i.e., the noise level multiplying the total number of parameters over the total number of samples. In contrast, $|S|$ may be zero, in which case $|V|^{-1}\|\widehat{\mtheta} - \mtheta^*\|_F^2\le 2\kappa^2\rho^2$. As discussed above, in order to yield smallest estimation error, $G$ must be as connected as possible within $\cc_i^* (1\le i\le K_*)$.

\subsubsection{Probabilistic Results}\label{sec:prob_result}
Under certain assumptions of the distribution and specific choice of $\phi(\cdot)$, we can have probabilistic bounds for $\rho^2$ and $\lambda^2$.

\begin{theorem}\label{thm:lam_rho}
Part (\rone) of Condition \ref{con:distribution} implies that, for some constant $C_{\rho}>0$,
\begin{align}\label{eq:rho_highp}
\rho^2 \leq  C_\rho\sigma^2\frac{K p \log(1/\xi)}{n |V|} 
\end{align}
with probability at least $1-\xi$, where $n=\min_{k\in V}n_k$. If we choose $\phi(\cdot)=\|\cdot\|_1$, part (\rone) of Condition \ref{con:distribution} also implies that, for some constant $C_{\lambda}>0$,
\begin{align}\label{eq:lambda_highp}
\lambda^2 \leq \left(\frac{C_\lambda\sigma^2}{\gamma_G^2}\right)\frac{p\log(|E|/\xi)}{n |V|}
\end{align}
with probability at least $1-\xi$, where $\gamma_G^2$ denotes the smallest nonzero eigenvalue of $\md^\t \md$. 
In particular, under the Condition \ref{con:identifiability}, part (\rone) and part (\rthree) of Condition \ref{con:distribution}, and Condition \ref{con:comp_graph}, choosing $\kappa = \max\{4\overline{\lambda}(\log 2)/ 3, 3/\underline{\lambda}\}$, it holds with probability at least $1-O(p |V|\exp\{-c n\}+ \xi)$ that
\begin{align}\label{eq: thm_prop}
\frac{1}{|V|}\|\widehat{\mtheta} - \mtheta^*\|_F^2&\le 2\kappa^2 \left\{C_\rho\sigma^2\frac{K p \log(1/\xi)}{n |V|}  + \biggl(\frac{4C_\lambda\sigma^2}{\kappa_0\gamma_G^2}\biggr)\frac{p|S|\log(|E|/\xi)}{n |V|}\right\}.
\end{align}   
\end{theorem}

From \eqref{eq: thm_prop}, the convergence rate of our estimator scales as $O_p(\sigma^2 p(K+|S|)/(n |V|))$ up to a logarithmic factor. We refer to the term 
\[
\gf_{G_0}(G)\equiv K(E_0)/(K(E)+|E\setminus E_0|)
\]
as the \emph{graph fidelity} of $G=(V,E)$ with respect to $G_0=(V,E_0)$ that measures the faithfulness of $G$ to $G_0$. Interestingly, there exists a phase transition of the performance of $\widehat{\mtheta}$ as the graph fidelity of $G$ with respect to $G_0$ varies from the minimal value to the maximal value. Suppose that $K(E_0)\neq 1$; otherwise there is no distribution heterogeneity among devices. In this case,
\[
\gf_{\min}\equiv \frac{2K(E_0)}{|V|^2\big(1-K^{-1}(E_0)\big)+2}\le\gf_{G_0}(G)\le 1.
\]
For $G$ such that $\gf_{G_0}(G)=1$ (e.g., $G=G_0$), convergence rate of the resulting estimator scales as $O_p(\sigma^2 pK(E_0)/(n |V|))$ which is the same as the \emph{oracle} estimator obtained by aggregating samples sharing the same distribution.  For $G$ such that $\gf_{G_0}(G)=K(E_0)/|V|$ (e.g., $E=\emptyset$ such that local devices do not communicate with each other), the associated estimator converges with the same rate $O_p(\sigma^2 p/n)$ as the local estimator defined in \eqref{eq:naive_est}. For $G$ such that $\gf_{G_0}(G)=\gf_{\min}$ (e.g., $G$ is complete), the resulting estimator has the convergence rate $O_p\big(\sigma^2p n^{-1}\{|V|(1-K^{-1}(E))\}\big)$, which is worse than the local estimator if $|V|>2$. Notice that, the global estimator obtained by treating $\vtheta_u^*\equiv \vtheta^* \ (\forall u\in V)$ is inconsistent when $K(E_0)\neq 1$, since it ignores distribution heterogeneity among local devices. 

Therefore, our method is most effective when the average size of cliques $|V|/K(E_0)$ is growing. Provided that $\gf_{G_0}(G)$ is close to $1$, the larger $|V|/K(E_0)$ is, the more substantial our method outperforms the local estimation \eqref{eq:naive_est}, which justifies our capability of dealing with a large number of heterogeneous devices simultaneously. In the subsequent section, we introduce a simple yet powerful method to enlarge the graph fidelity of any given graph $G$.


\subsection{Edge Selection by multiple testing}\label{sec:edge_selection}
In order to improve the performance of our method, we propose to find a subgraph $\widehat{G}=(V,\widehat{E})$ of $G=(V,E)$ with the largest graph fidelity with respect to $G_0$,
\begin{align}\label{eq:op_graph}
\widehat{E}\in \argmin_{\widetilde{E}\subset E} \left\{K(\widetilde{E}) + |\widetilde{E}\setminus E_0|\right\}.
\end{align}
Note that $\widehat{G}=(V,\widehat{E})$ represents the graph which gives rise to the best estimator based on $G=(V,E)$. If there exist multiple connected components in $G$, the objective function of \eqref{eq:op_graph} is separable according to the connected components. Without loss of generality, we assume $G$ is connected such that $K(E)=1$. 
The following proposition relates problem \eqref{eq:op_graph} to selecting true edges in $E$. 

\begin{proposition}\label{prop:optimal_graph_mht}
For any graph $G=(V,E)$ with $K(E)=1$ and $E_0$ in Definition \ref{def:data_dist}, we have
\[
\min_{\widetilde{E}\subset E}\{K(\widetilde{E})+|\widetilde{E}\setminus E_0|\}=K(E\cap E_0),
\]
where $K(E\cap E_0)$ denotes the number of connected components of $G=(V,E\cap E_0)$.
For any $\widetilde{E}\subset E$,
\begin{align}\label{eq:upper_bound_c}
K(E\cap E_0)\le K(\widetilde{E}) + |\widetilde{E} \setminus E_0|\le K(E\cap E_0) + |\widetilde{E} \setminus E_0| + |(E\cap E_0) \setminus \widetilde{E}|.
\end{align}
\end{proposition}
Proposition \ref{prop:optimal_graph_mht} implies that one of the minimizers of \eqref{eq:op_graph} is exactly $E\cap E_0$, which suggests to find a $\widetilde{E}\subset E$ with a small $|\widetilde{E} \setminus E_0| + |(E\cap E_0) \setminus \widetilde{E}|$. 
By Definition \ref{def:data_dist}, $(i,j)\in E_0$ is equivalent to $\vtheta_i^*=\vtheta_j^*$.
This motivate us to consider the simultaneous testing of the following null hypotheses
\begin{align}\label{eq:mht}
H_{0,e}\colon \vtheta_{e^+}^* = \vtheta_{e^-}^*\quad \mathrm{versus} \quad H_{1,e}\colon \vtheta_{e^+}^* \neq \vtheta_{e^-}^*,\quad e\in E.
\end{align}
We impose Condition \ref{con:2ndbartlett} for technical convenience.
\begin{condition}\label{con:2ndbartlett}
For each $u\in V$, $\msigma_u(\vtheta_u^*) = \rmh_u(\vtheta_u^*)$. 
\end{condition}
The requirement $\msigma_i(\vtheta_i^*) = \rmh_i(\vtheta_i^*)$, which is also known as the second Bartlett identity \citep{bartlett1953approximate}, is met if the probability density function, $p_u(\vz;\vtheta)$, of $P(\vtheta_{u}^*)$ is uniformly integrable with respect to $\vz$ over $\vtheta\in\sxi$.            
As shown in Lemma \ref{lem:naive_consistency}, the local estimator $\widehat{\vtheta}_u^{\loc}$  defined in \eqref{eq:naive_est} is asymptotically normal with asymptotic mean $\vtheta_u^*$, and variance $\rmh_u(\vtheta_u^*)^{-1}\msigma_u(\vtheta_u^*)\rmh_u(\vtheta_u^*)^{-\t}$. Condition \ref{con:2ndbartlett} is used to obtain a consistent estimator of the asymptotic variance of $\widehat{\vtheta}_u^{\loc}$.
Thus, for $\vtheta_{e^+}^* = \vtheta_{e^-}^*$, we can construct a test statistic via
\begin{align}\label{eq:test_statistic_def}
\widehat{\rcw}_e = \Big\{\big(\widehat{\vtheta}_{e^+}^{\loc} - \widehat\vtheta_{e^{-}}^{\loc}\big)^{\t}\big(\widehat{\mupsilon}_{e^+} + \widehat{\mupsilon}_{e^-}\big)^{-1}\big(\widehat{\vtheta}_{e^+}^{\loc} - \widehat\vtheta_{e^{-}}^{\loc}\big)\Big\}^{1/2},
\end{align}
where $\widehat{\mupsilon}_u=\{n_u \widehat{\rmh}_u(\widehat{\vtheta}_u^{\loc})\}^{-1}$ denotes the asymptotic variance of $\widehat{\vtheta}_u^\loc$.
Adopting Bonferroni correction, we select $E\cap E_0$ by
\begin{align}\label{eq:bonferroni}
\widehat{E} = \big\{e\in E : |\widehat{\rcw}_e|^2\le \chi_{p}^2(\alpha/|E|)\big\},
\end{align}
where $\chi_{p}^2(\alpha)$ is the upper $\alpha$-quantile of the $\chi_p^2$ distribution. For $\vtheta_{e^+}^* \neq \vtheta_{e^-}^*$, to adaptively measure the distance between them, define the distance $\mathrm{dist}(\vtheta_1,\vtheta_2)=\big\{\big(\vtheta_1 - \vtheta_2\big)^{\t}\big(c_{e^+}\mupsilon_{e^+}^*+c_{e^-}\mupsilon_{e^-}^*\big)^{-1}\big(\vtheta_1 - \vtheta_2\big)\big\}^{1/2}$, where $c_{u}=\lim_{n_e\to +\infty}n_e/ n_{u}\ (n_e=\min\{n_{e^+},n_{e^-}\})$ and $\mupsilon_u^*=\{\rmh(\vtheta_u^*)\}^{-1}$ for $u\in\{e^+,e^-\}$. Under certain minimum signal condition, we show that our procedure can consistently select edges in $E_0$ with a large probability.

\begin{theorem}\label{thm:sel_con}
Under the same conditions imposed in Proposition \ref{prop:asymptotic_normality}, if further
\begin{align}\label{con:min_signal}
\min_{e\in E\setminus E_0} n_e\Big\{\mathrm{dist}(\vtheta_{e^+}^*,\vtheta_{e^-}^*)\Big\}^2 \ge 4\chi^2_{p}(\alpha/|E|),
\end{align}
then $
\liminf_{n\to\infty} P\left(\widehat{E} = E\cap E_0\right) \ge 1 - \alpha$.
\end{theorem}
Theorem \ref{thm:sel_con} demonstrates that, our selection procedure \eqref{eq:bonferroni} incurs no false negatives and false positives, with confidence level $1-\alpha$. Remarkably, this edge selection procedure does not require data exchange between devices. Theorem \ref{thm:sel_con} also suggests that the graph that gives rise to the optimal convergence rate is not necessarily $G_0$, but any graph $G=(V,E)$ such that $K(E\cap E_0)=K(E_0)$, revealing that our method is robust to graph mis-specification. Nonetheless, \eqref{eq:bonferroni} is conservative for controlling the false positive rate \citep{benjamini1995controlling}. Despite that we 
have obtained the asymptotic distribution of the test statistic in \eqref{eq:test_statistic_def}, the dependency among testing $H_{0,e}, e\in E$ poses a challenge to controlling FPR of our edge selection procedure, which is left for future work. 


\section{Decentralized Stochastic ADMM}\label{sec:optim}
In this section, we focus on the problem of solving \eqref{eq:ob}. In fact, \eqref{eq:ob} is has a similar form to trend filtering \citep{JMLR:wang2016} and convex clustering \citep{LindstenOL:2011,Tang2016fused}, which can be efficiently solved by the primal-dual interior-point
method \citep{kim2009}, (stochastic) first-order primal-dual optimization algorithms \citep{Ho2019GlobalEB}, and Alternating Direction Method of Multipliers (ADMM) \citep{aaditya2016}. However, none of them is directly applicable when data can not be shared across devices. \cite{hallac2015network} proposed a decentralized version of ADMM to optimize \eqref{eq:ob}, where only parameters are transmitted across edges after being updated on each local device. However, it requires to use all the data on each device in each iteration, and thus is not appropriate for online settings, or devices with weak computational capacity.

We propose an algorithm called Fed-ADMM to solve the optimization problem \eqref{eq:ob}. Fed-ADMM is a decentralized stochastic version of ADMM. Apart from not transmitting local data, only a mini-batch of samples is used in each iteration on each device, and we allow local optimization speeds and availability patterns to be heterogeneous across devices. Without loss of generality, we assume that edges in $G$ point from the larger node towards the smaller node such that $(i,j)\in E$ implies $i>j$. Let $N_i=\{j: (i,j)\in E\}\cup\{j:(j,i)\in E\}$ denote the neighbors of node $i$.  

We first consider the case where all devices are available instantaneously. Similar to \cite{hallac2015network}, we introduce auxiliary vectors $\vbeta_{i j}, \vbeta_{j i}$ with the constraints $\vbeta_{i j}=\vtheta_i, \vbeta_{j i}= \vtheta_j$, for all $ (i ,j)\in E$. The augmented Lagrangian \citep{hestenes1969multiplier} is
\begin{equation}\label{eq:lagrange}
\begin{aligned}
& L(\mtheta,\mbeta,\malpha)= \frac{1}{|V|}\sum_{i\in V}\widehat{M}_i(\vtheta_i) + \lambda \sum_{(i,j)\in E}\phi(\vbeta_{i j} - \vbeta_{j i}) \\
&- \sum_{(i,j)\in E}\left\{ \valpha_{i j}^\t (\vtheta_i - \vbeta_{i j}) + \valpha_{j i}^\t (\vtheta_j - \vbeta_{j i}) \right\} + \frac{\rho}{2}\sum_{(i,j)\in E} \left\{\|\vtheta_i - \vbeta_{i j}\|_2^2 + \|\vtheta_j - \vbeta_{j i}\|_2^2\right\},
\end{aligned}
\end{equation}
where $\mbeta = (\vbeta_{i j},\vbeta_{j i}: (i,j)\in E)$ and $\malpha = (\valpha_{i j},\valpha_{j i}: (i,j)\in E)$. Generally, ADMM solves \eqref{eq:lagrange} iteratively by minimizing
$L(\mtheta,\mbeta,\malpha)$ with respect to $\mtheta$ and $\mbeta$ alternatively given the other fixed,
followed by an update over the Lagrangian multiplier $\malpha$. Notably, $L(\mtheta,\mbeta,\malpha)$ is separable, and updates for $\mtheta,\mbeta$, and $\malpha$ can be executed in a distributed way. 

In practice, however, local devices cannot afford to optimize with the whole dataset. Motivated by stochastic gradient descent, instead of directly minimizing $L(\mtheta,\mbeta(t),\malpha(t))$ with respect to $\vtheta_i$ on device $i$, we adopt one-step stochastic gradient update in the $t$-th iteration:
\begin{equation}\label{eq:ADMM_primal}
\begin{aligned}
\vtheta_i(t+1) = \vtheta_i(t) - &\eta(t)\left\{\widetilde{\rvg}_i(t)+ \rho\sum_{j\in N_i} (\vtheta_i(t) - \vbeta_{i j}(t) - \rho^{-1}\valpha_{i j}(t)) \right\},
\end{aligned}
\end{equation}
where $\widetilde{\rvg}_i(t) = |\cb_i(t)|^{-1}\sum_{b\in \cb_i(t)}\vpsi_{i}(\rvz_b^{(i)};\vtheta_i(t))$, $\eta(t)$ denotes the learning rate, $\cb_i(t)$ denotes the mini-batch randomly sampled from $\{\rvz_k^{(i)}\}_{k=1}^{n_i}$ on device $i$ in the $t$-th iteration, and $\widetilde{\rvg}_i(t)$ is an unbiased estimator of $\nabla_{\vtheta}\widehat{M}_i(\vtheta)$ evaluated at $\vtheta_i(t)$. Except for local samples, the update equation \eqref{eq:ADMM_primal} only requires $\vbeta_{i j}(t)$ and $\valpha_{i j}(t)$ which can be transmitted from device $j$. Thus, \eqref{eq:ADMM_primal} can be executed in parallel for all devices. 

With $\vtheta_i(t+1), i\in V$ at hand, we then update $\vbeta_{i j}(t)$ and $\vbeta_{j i}(t)$ by
\begin{equation}\label{eq:update_beta}
\begin{aligned}
&\biggl(\begin{matrix}\vbeta_{i j}(t+1)\\ \vbeta_{j i}(t+1)\end{matrix}\biggr)= \argmin_{\vbeta_{i j},\vbeta_{j i}}\biggl\{ \lambda\phi(\vbeta_{i j} - \vbeta_{j i}) \biggr. \\
&\quad \quad \quad\quad\biggl.+\frac{\rho}{2}\biggl(\|\vtheta_i(t+1)-\vbeta_{i j} - \rho^{-1}\valpha_{i j}(t)\|_2^2 +
\|\vtheta_j(t+1)-\vbeta_{j i} - \rho^{-1}\valpha_{j i}(t)\|_2^2\biggr) \biggr\}.
\end{aligned}
\end{equation}

The update equation \eqref{eq:update_beta} can be implemented on either device $i$ or device $j$, as long as $(\vtheta_j(t+1), \vbeta_{j i}(t), \valpha_{j i}(t))$ or $(\vtheta_i(t+1), \vbeta_{i j}(t), \valpha_{i j}(t))$ is transmitted to the corresponding device. We highlight that, for specific choice of $\phi(\cdot)$, for example, $\phi(\cdot) = \|\cdot\|_1$ and $\phi(\cdot) = \|\cdot\|_2$, we can obtain an explicit update equation from \eqref{eq:update_beta}. Finally, we update $\valpha_{i j}(t)$ and $\valpha_{j i}(t)$ by
\begin{equation}\label{eq:update_alpha}
\biggl(\begin{matrix}\valpha_{i j}(t+1)\\ \valpha_{j i}(t+1)\end{matrix}\biggr) = \biggl(\begin{matrix}\valpha_{i j}(t)\\ \valpha_{j i}(t)\end{matrix}\biggr) - \rho\biggl(\begin{matrix}\vtheta_i(t+1) - \vbeta_{i j}(t+1)\\ \vtheta_j(t+1) - \vbeta_{j i}(t+1)\end{matrix}\biggr).
\end{equation}
Notice that update equation \eqref{eq:update_alpha} also only requires parameter communication among connected devices. Both \eqref{eq:update_beta} and \eqref{eq:update_alpha} can be performed in parallel across edges. We refer to \eqref{eq:ADMM_primal} as node optimization step, and \eqref{eq:update_beta} and \eqref{eq:update_alpha} as edge communication step. Fed-ADMM is summarized in Algorithm \ref{alg:admm}.

Our algorithm bears a resemblance to \citet{ouyang2013stochastic} and \citet{suzuki2013dual}, who approximated $m_{i}(\cdot;\vtheta_i)$ with a linear function $m_{i}(\cdot;\vtheta_i(t)) + (\vtheta - \vtheta_i(t))^\t\vpsi_{i}(\cdot;\vtheta_i(t))$, and used the proximal method \citep{rockafellar1976augmented} to update $\vtheta_i(t)$,
\begin{align}
\vtheta_i(t+1)= \argmin_{\vtheta\in\sxi}&\biggl\{ \vtheta^\t\widetilde{\rvg}_i(t)+ \frac{\rho}{2}\sum_{j\in N_i}\|\vtheta-\vbeta_{i j}(t) - \rho^{-1}\valpha_{i j}(t)\|_2^2
+\frac{\|\vtheta - \vtheta_i(t)\|_2^2}{2\widetilde{\eta}(t)} \biggr\},\label{eq:ouyang_update_theta}
\end{align}
where $\widetilde{\eta}(t)$ is the step-size. Different from \eqref{eq:ADMM_primal}, their work was motivated by the linearized proximal point method \citep{xu2011class}. Interestingly, \eqref{eq:ADMM_primal} can be viewed as an extension of \eqref{eq:ouyang_update_theta} with an adaptive learning rate. In more specific, choosing ${\eta}(t) = \widetilde\eta(t)/(1 + \rho|N_i|\widetilde\eta(t))$, one can show that \eqref{eq:ouyang_update_theta} is equivalent to \eqref{eq:ADMM_primal}. 
\begin{algorithm}[htp] \SetKwInOut{Input}{Input}\SetKwInOut{Output}{Output} 
\Input{Initial value $\mtheta(0),\mbeta(0),\malpha(0)$, number of iterations $T$, and the learning rate $\eta(t), 1\le t\le T$.}
\Repeat{$t> T$}{ 
Sample mini-batches $\cb_i(t)$ on device $i$ in parallel\; 
Obtain $\vtheta_i(t + 1)$ on device $i$ by \eqref{eq:ADMM_primal} in parallel for each $i\in V$\;
Broadcast each $\vtheta_i(t + 1)$ to neighbor devices\;
Obtain $\vbeta_{i j}(t + 1)$ and $\vbeta_{j i}(t+1)$ on device $i$ \footnotemark{}  by \eqref{eq:update_beta} in parallel for $(i,j)\in E$\;
Obtain $\valpha_{i j}(t + 1)$ and $\valpha_{j i}(t+1)$ on device $i$ by \eqref{eq:update_alpha} in parallel for $(i,j)\in E$\;
Broadcast $(\vbeta_{i j}(t+1),\vbeta_{j i}(t+1))$ and $(\valpha_{i j}(t+1),\valpha_{j i}(t+1))$ from device $i$ to device $j$ in parallel for $(i,j)\in E$\;
$t \leftarrow t + 1$
}
\Output{$\overline{\mtheta}=T^{-1}\sum_{t=1}^{T}\mtheta(t-1)$}
\caption{Decentralized stochastic ADMM}
\label{alg:admm} 
\end{algorithm}
\footnotetext{In fact, we can obtain $\vbeta_{i j}(t + 1)$ and $\vbeta_{j i}(t+1)$ on either device $i$ or device $j$. Here, we choose device $i$ as the working device without loss of generality. }

In the sequel, we show that the output of Algorithm \eqref{alg:admm} converges to the global minimizer of \eqref{eq:lagrange}.  Let $\big(\vtheta_i(t), \vbeta_{i j}(t), \vbeta_{j i}(t), \valpha_{i j}(t),\valpha_{j i}(t)\big)$ be the output of Algorithm \ref{alg:admm} in the $t$-th iteration, $0\le t\le T$.
Denote by $\big\{\widehat\vtheta_i,\widehat\vbeta_{i j},\widehat\vbeta_{j i},\widehat\valpha_{i j},\widehat\valpha_{j i}; i,j\in V\big\}$ the global minimizer of \eqref{eq:lagrange}. Without loss of generality, we assume that $\big\{\widehat\vtheta_i,\widehat\vbeta_{i j},\widehat\vbeta_{j i}, \vtheta_i(t), \vbeta_{i j}(t), \vbeta_{j i}(t); i,j\in V, 0\le t\le T\big\} \subset \sxi$, since we can always project them into $\sxi$.
\begin{theorem}\label{thm:conv_alg_admm}
Let $\kappa_{\alpha} = \max_{i\in V, j\in N_i} (\|\valpha_{i j}(0)\|_2^2+\|\widehat{\valpha}_{i j}\|_2^2)$. Suppose that $\kappa_\alpha\ge \lambda \sup_{\va\neq \0}\phi(\va)\linebreak\|\va\|_2^{-1}$, $C_\psi=\max_{i\in V}\sup_{\vtheta\in\sxi} {n_i}^{-1}\sum_{k=1}^{n_i}\|\vpsi_i(\rvz_k^{(i)};\vtheta)\|_2^2<\infty$ and $r_0=\sup_{\va\in\sxi}\|\va\|_2<\infty$.
Under Condition \ref{con:identifiability} and part (\rtwo) of Condition \ref{con:distribution}, by choosing $\eta(t) = \kappa / t$, we have
\[
\frac{1}{|V|}E\left(\left.\|\overline{\mtheta}-\widehat{\mtheta}\|_F^2\right|\rvz_k^{(u)}, 1\le k\le n_u, u\in V\right)\le \frac{2\kappa^2 C_{\psi}\log T}{T},
\]
for sufficiently large $T$ such that $\kappa C_{\psi} |V|\log T \ge |E|(8\rho^{-1}\kappa_{\alpha} + 2r_0^2 \rho + 4r_0\kappa_\alpha)$, where the expectation is taken with respect to the choice of mini-batches $\{\cb_u(t)\colon u\in V, 1\le t\le T\}$.
\end{theorem}
\subsection{Extension of Fed-ADMM to Heterogeneous Accessibility of Devices }\label{sec:fedadmm_off}
Now we consider the case where a proportion of devices can be inaccessible during the training process. Let $R_i(t)=\1\{\mathrm{device}\ i\ \mathrm{is\ available\ in\ the\ iteration\ }t\}$ for $i\in V$, $t\in\mathbb{N}$ and $S(t)=\{i: R_i(t)=1\}$ be the set of available devices in the $t$-th iteration. We model the inaccessibility by imposing randomness on $R_i(t)$. Suppose that $R_i(t)$ is a Bernoulli random variable with mean $p_i$ which is independent of $t$. We assume that $R_i(t), t>0$ enjoys the memoryless property, \ie, $\{R_i(t_1); i\in V\}$ is independent of $\{R_i(t_2); i\in V\}$ for any $t_1\neq t_2$. We allow a heterogeneous offline rate among devices, \ie, $p_i\neq p_j$ for $i\neq j$, and allow the dependency among $(R_i(t), i\in V)$ for a fixed $t$. Let $p_0=\min_{i\in V} p_i>0$.

To begin with, we presume the existence of a central machine $\mathcal{S}$ and known $p_i, i\in V$. In the $t$-th iteration, for the node computation step, device $i$ needs local samples to produce an unbiased estimator of $\nabla_{\vtheta_i} \widehat{M}_i(\vtheta_i(t))$. 
Motivated by the inverse probability weighting method \citep{wooldridge2007inverse,mansournia2016inverse} in the causal inference literature, an unbiased estimator of $\nabla_{\vtheta_i} \widehat{M}_i(\vtheta_i(t))$ is
$|\cb_i(t)|^{-1}\sum_{b\in \cb_i(t)}p_i^{-1}\vpsi_{i}(\rvz_b^{(i)};\vtheta_i(t))R_i(t)$ owing to the independence between $R_i(t)$ and $\cb_i(t)$.
Therefore, we modify \eqref{eq:ADMM_primal} 
as
\begin{subequations}
\begin{align}
&\vtheta_i(t+1) = \vtheta_i(t) - \eta(t)\left\{\frac{1}{|\cb_i(t)|}\sum_{b\in \cb_i(t)}p_i^{-1}\vpsi_{i}(\rvz_b^{(i)};\vtheta_i(t))\right.\notag\\
&\quad\quad\quad\quad\quad\quad\quad\left.+ \rho\sum_{j\in N_i} (\vtheta_i(t) - \vbeta_{i j}(t) - \rho^{-1}\valpha_{i j}(t)) \right\}\quad\mathrm{on\ device}\ i\ \mathrm{with}\ R_i(t) = 1;\label{eq:update_theta_online}\\
&\vtheta_i(t+1) = \vtheta_i(t) - \eta(t)\rho\sum_{j\in N_i} (\vtheta_i(t) - \vbeta_{i j}(t) - \rho^{-1}\valpha_{i j}(t)) \quad\mathrm{on}\ \mathcal{S}\ \mathrm{with}\ R_i(t) = 0,\label{eq:update_theta_offline}
\end{align}
\end{subequations}
where for \eqref{eq:update_theta_online}, we need to send $\{\vbeta_{i j}(t),\valpha_{i j}(t); j\in N_i\}$ from $\mathcal{S}$ to device $i$, and then send $\vtheta_i(t+1)$ back to $\mathcal{S}$ from device $i$. This procedure applies to all devices in $S(t)$. After $\vtheta_i(t+1) (i\in V)$ being updated and transmitted to $\mathcal{S}$, we can derive $\big(\vbeta_{i j}(t+1),\vbeta_{j i}(t+1)\big)$ by \eqref{eq:update_beta} and $\big(\valpha_{i j}(t+1),\valpha_{j i}(t+1)\big)$ by \eqref{eq:update_alpha} on $\mathcal{S}$.

Similar to Theorem \ref{thm:conv_alg_admm}, we obtain the convergence rate of Algorithm \ref{alg:admm_offline} in Supplementary materials.
\begin{corollary}\label{coro:conv_alg_admm}
Under the same conditions of Theorem \ref{thm:conv_alg_admm}, for the output of Algorithm \ref{alg:admm_offline} with known $\{p_i, i\in V\}$ with $p_0=\min_{i\in V}p_i>0$, we have
\[
\frac{1}{|V|}E\left(\left.\|\overline{\mtheta}-\widehat{\mtheta}\|_F^2\right|\rvz_k^{(u)}, 1\le k\le n_u, u\in V\right)\le \frac{2\kappa^2 C_{\psi}\log T}{p_0 T},
\]
for sufficiently large $T$, where the expectation is taken with respect to the choice of mini-batches $\{\cb_u(t)\colon u\in V, 1\le t\le T\}$.
\end{corollary}

\section{Simulations}\label{sec:simulation}
In this section, we evaluate the performance of five methods, Fed-ADMM, Fed-ADMM-ES, Oracle, Local and Global in various settings. Here, for a specific graph $G=(V,E)$, Fed-ADMM denotes the output of Algorithm \eqref{alg:admm} with $G$, Fed-ADMM-ES denotes the output of Algorithm \eqref{alg:admm} after applying the adaptive edge selection procedure to $G=(V,E)$, Oracle denotes the output of Algorithm \eqref{alg:admm} with the characteristic graph $G_0$, Local denotes the local estimator defined in \eqref{eq:naive_est}, and Global denotes the estimator obtained by ignoring the heterogeneity, \ie,
\[
\widehat{\mtheta}^{\mathrm{gl}} = \argmin_{\mtheta} \frac{1}{|V|}\sum_{u\in V}\sum_{k=1}^{n_u} m_u(\rvz_k^{(u)};\vtheta_u)\quad\mathrm{s.t.}\ \vtheta_i=\vtheta_j, \forall \{i,j\}\subset V.
\]
The metric of performance is chosen as the average squared estimation error, i.e.,  $\|\widehat{\mtheta}-\mtheta^*\|_F^2/|V|$.
We also compare the algorithmic convergence rates of Fed-ADMM and vanilla stochastic gradient descent (SGD) for solving the objective function \eqref{eq:penalized_m}, \ie, total number of iterations required to converge. 

We first introduce data generating processes of our simulation. We consider linear regression tasks on each device, where covariates $\rvx_k^{(u)}\sim N_p(0,\mi_p)$ and noises $\varepsilon_k^{(u)}\sim N(0,1)$ independently for $k=1,\ldots,n_u; u\in V$. The characteristic graph $G_0$ is generated by evenly partitioning $V$ into $K_0$ subsets, $V_1,\ldots,V_{K_0}$, and constructing a complete subgraph with node set being $V_j$. For each $V_j$ and $u\in V_j$, the responses $\ry_k^{(u)} = \big(\rvx_k^{(u)}\big)^\t \vvartheta^{(j)} + \varepsilon_k^{(u)}$, where $\vvartheta^{(j)}\ (1\le j\le K_0)$ are sampled independently from a Gaussian distribution with mean $0$ and covariance matrix $p^{-1/2}\mi_p$. We store the adjacency matrix $\mlambda_0$ of $G_0$.

\begin{figure}[htb]
\centering
\includegraphics[width=\textwidth]{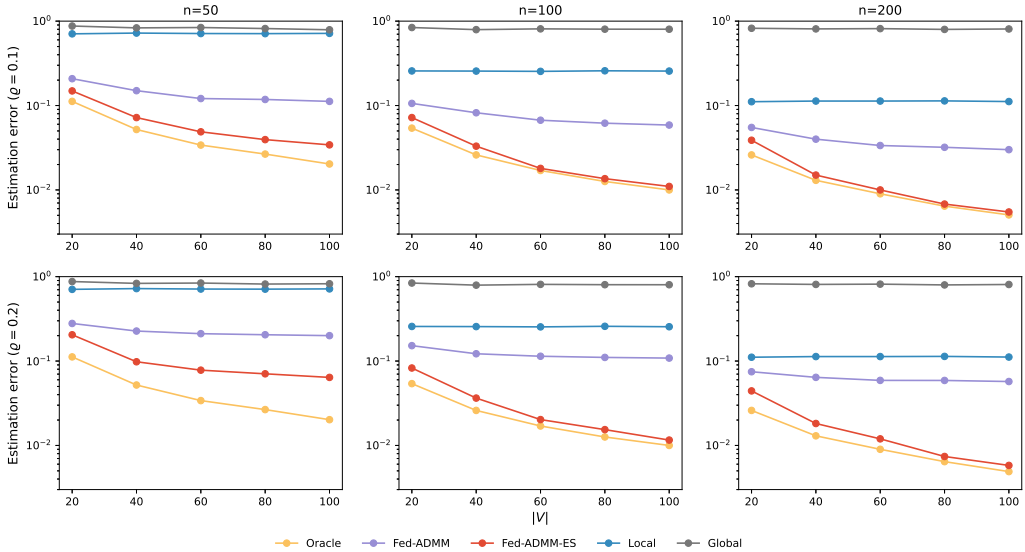}
\caption{Average estimation error of Fed-ADMM, Fed-ADMM-ES and oracle, local and global estimators in linear regression, with randomly corrupted graphs. The number of clusters $K$ is fixed to be $5$ for all settings. The two rows correspond to corruption level $\varrho=0.1$ and $0.2$.}
\label{fig:S1}
\end{figure}

We then generate $G=(V,E)$ by maliciously corrupting $G_0$ with \textit{corruption level} $\varrho>0$. By sampling Bernoulli random variables $\re_{i j}$ independently with mean $\varrho$, we flip the connection status between device $i$ and device $j$ in $G_0$ if $\re_{i j}=1$; otherwise keep it intact. In more specific, we directly generate the adjacency matrix $\mlambda(\varrho)$ by $
\mlambda_{i j}(\varrho)=\mlambda_{j i}(\varrho) = 
\re_{i j}\{1 - (\mlambda_0)_{i j}\} + (1-\re_{i j})(\mlambda_0)_{i j}$
for $i, j\in V$ with $i<j$, and choose $G$ as the graph that $\mlambda(\varrho)$ defines. The deviation between $G$ and $G_0$ can be adjusted by different choices of $\varrho$. Indeed, $|E\setminus E_0|+ |E_0\setminus E| = \sum_{i< j}\re_{i j}=\varrho|V|(|V| - 1) / 2 + O_p(\varrho|V|\log|V|)$ by Hoeffding's inequality. 

We choose $m_u(\rvz_k^{(u)};\vtheta)=\big(\ry_k^{(u)}-\big(\rvx_k^{(u)}\big)^\t \vtheta \big)^2/2$ and $\phi(\cdot)=\|\cdot\|_1$ in \eqref{eq:ob}. The regularization parameter $\lambda$ is tuned by cross-validation.

\subsection{Estimation Error}
We compare the averaged estimation error of aforementioned five estimators. For simplicity, we set $n_i = n$ for all $i\in V$. We choose $|V|\in \{20, 40, 60, 80, 100\}$, $n\in \{50, 100, 200\}$, and fixed $K=5$ and $p=20$. We then conduct simulations with corruption level $\varrho=0.1$ and $\varrho=0.2$ respectively.

In Figure \ref{fig:S1}, we report the average squared estimation error for each estimator. The first row and the second row correspond to the corruption level $\varrho=0.1$ and   $\varrho=0.2$, respectively.  Each point is the mean $100$ of independent replications. The performance of Local and Global do not depend on $|V|$ in all settings, while the estimation error of Oracle is decreasing as $|V|$ increases. In the case $\varrho=0.1$, the Fed-ADMM has smaller estimation errors than local estimators, which shows the superiority of data federation. However, there still exists a non-vanishing performance gap between Fed-ADMM and the oracle estimator. Notably, the average error of Fed-ADMM does not decrease as $|V|$ increases, since the corruption level $\varrho$ is fixed as $|V|$ grows and thus the expected number of wrong edges increases as well, which hinders the performance as shown in Theorem \ref{thm:lam_rho}. Fed-ADMM performs even worse than the local estimator as shown in the second row of Figure \ref{fig:S1} if we use an even worse graph.

We highlight that this issue can be solved efficiently by edge selection. In all settings, Fed-ADMM-ES (Fed-ADMM with edge selection) outperforms Fed-ADMM and local estimator, and is almost indistinguishable from Oracle when $n$ is large. This suggests that most of the misleading information in graph can be eliminated by the edge selection procedure proposed in Section \ref{sec:edge_selection}.

\subsection{Sensitivity Analysis of Graphs}
\begin{figure}[htbp]
\centering
\includegraphics[width=\textwidth]{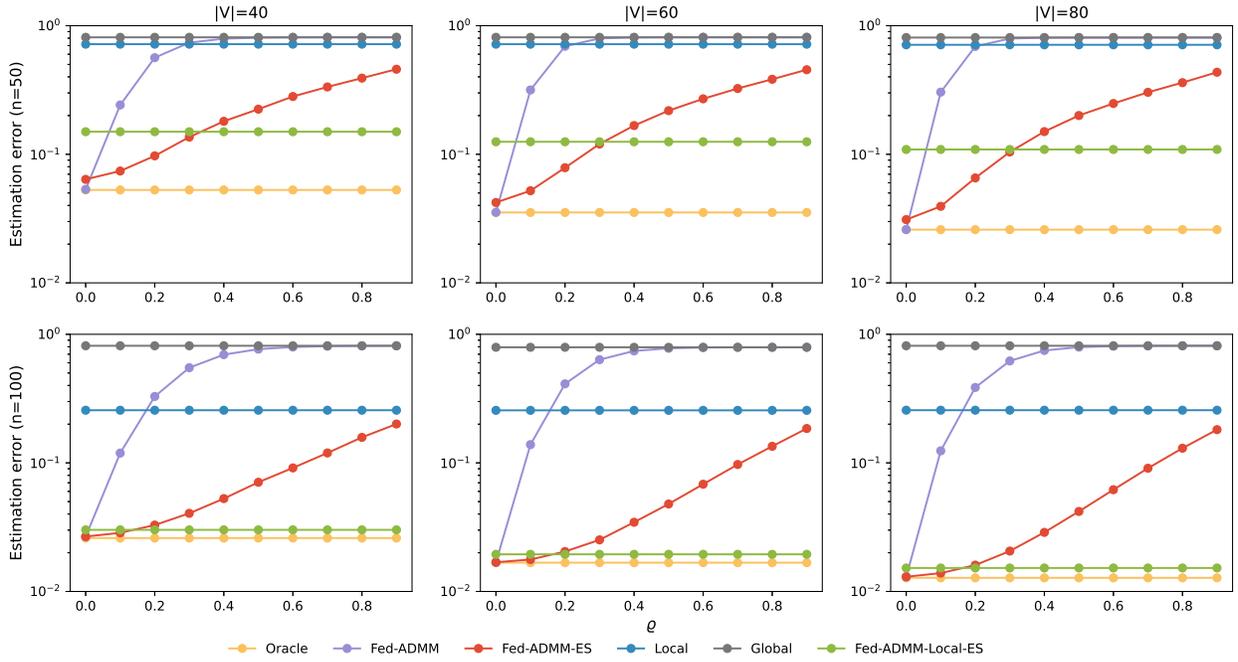}
\caption{Average estimation error of Fed-ADMM, Fed-ADMM-ES, Fed-ADMM-Local-ES, Oracle, Local and Global in linear regression, with randomly corrupted graphs. The $x$-axis corresponds to the corruption lever $\varrho$. We fix $K=5$. The two rows correspond to $n=50$ and $100.$}
\label{fig:S3}
\end{figure}

In order to further study how the performance is degraded of Fed-ADMM when $\varrho$ increases, we let $\varrho$ vary from $0$ to $0.9$ in steps of $0.1$, and compare the averaged estimation error for Fed-ADMM and Fed-ADMM-ES. Besides, we also consider edge selection based only on local estimators (Fed-ADMM-Local-ES). That is, we do not use the information of the given graph, but only use local estimators on each device, to do edge selection. Notice that the estimation error of Fed-ADMM-Local-ES does not depend on $\varrho$, no matter how the given graph $G$ can be misleading. Meanwhile, this method cannot benefit from $G$ when it is close to $G_0$.

We set $K=5$, and $|V|\in\{40, 60, 80\}$ and $n\in\{50, 100\}$. The results are reported in Figure \ref{fig:S3}. Each point is summarized from $100$ independent replications.
The average estimation error of Fed-ADMM increases rapidly as $\varrho$ increases and exceeds the local estimator when $\varrho\geq 0.2$. In contrast, the performance of Fed-ADMM-ES is much more robust. Even when $\varrho=0.9$, \ie. the adjacency matrix is almost completely misleading, it is still better than the local estimator. Besides, as expected, when $\varrho$ is small, which means that $\ma(\varrho)$ is close to the true adjacency matrix $\ma_0$, Fed-ADMM-ES is better than Fed-ADMM-Local-ES; otherwise, using local estimators to construct incidence matrix is a more reasonable choice.

\subsection{Algorithmic Convergence Rates}

The algorithmic convergence rates of Fed-ADMM and its variant are studied in this section. Vanilla gradient descent (GD) and stochastic gradient descent (SGD), which have convergence rates of order $1/\sqrt{T}$ in this convex but nonsmooth setting, are also considered for comparison. For $K=5$, $p=20$, $n\in\{50, 100\}$ and $|V|\in\{40,60,80\}$, we run the Fed-ADMM with full batch, Fed-ADMM with batch size $10$, GD and SGD with batch size $10$. The results are reported in Figure \ref{fig:S2}. The $x$-axis and $y$-axis correspond to the number of optimization iterations and average estimation error, respectively. In all settings, the convergence of Fed-ADMM (Fed-ADMM with full batch) is much faster than SGD (GD).

\section{A Real-Data Study}\label{sec:real}
In this section, we use our proposed method to study the 2020 US presidential election results.
The 2020 county-level election results is available from \url{https://github.com/tonmcg/US_County_Level_Election_Results_08-20} and the county-level information can be derived from \url{ https://www.kaggle.com/benhamner/2016-us-election}. The original data includes 51 states, 3111 counties and 52 county-level predictors. We treat each state as a device and its counties as its samples. We use the county-level information as the predictors and the election result of each county as responses. If Democrats win, the label of this county is encodes as $1$, otherwise $0$. We then use logistic regression to predict election results. Due to limited data availability, not all states have large enough sample size. We  then select the states with more than 50 counties into our study, and the total number of selected states is $29$. 

We consider the following two approaches to obtain the graph used in Fed-ADMM: (a) use the historical election results prior to 2016 to classify the states, denoted by $\widehat{\rmd}_{\his}$; (b) use the local estimators of states to do edge selection, denoted by $\widehat{\rmd}_{\loc}$. Here, $\widehat{\rmd}_{\his}$ is generated by the division of red, blue, and swing states. 
We use the local estimator and global estimator as baselines.

We randomly select $2/3$ counties as training data and the rest counties serve as test samples to measure their prediction accuracy, where accuracy is defined as the proportion of correctly classified samples over the whole test samples. We repeat this for $50$ times to report mean and standard deviation of each model's accuracies in Table \ref{tab:acc}.   
\begin{table}[h]
	\centering
	\fontsize{8}{11}\selectfont    
	\caption{Accuracy (mean(standard deviation)) of Local, Global, and FedADMM.  }
	\begin{tabular}{ccccc}
		\toprule
		\multicolumn{1}{c}{\multirow{2}{*}{Methods}} &  \multirow{2}{*}{Local} & \multirow{2}{*}{Global} &
		\multicolumn{2}{c}{FedADMM} \cr
		\cmidrule(lr){4-5}
		& & &  $\widehat{\rmd}_{\loc}$  & $\widehat{\rmd}_{\his}$ \cr
		\cmidrule(lr){1-5}
    	Accuracy & 0.741(0.034) & 0.752(0.012) & 0.793(0.019) & 0.742(0.011)    \cr
		\bottomrule
	\end{tabular}\vspace{0cm}
	\label{tab:acc}
\end{table}

Table \ref{tab:acc} shows that FedADMM with $\widehat{\rmd}_{\loc}$ performs best. Global estimator outperforms the local one, suggesting that the heterogeneity among the considered states is not strong. FedADMM with $\widehat{\rmd}_{\his}$ performs on par with local estimator, which indicates that the heterogeneity does not mainly come from the division of red, blue and swing states.

Despite the prediction performance, we are also interested in the graph obtained by edge selection, which reflects similarity between states in 2020 presidential election. We plot the clustering result about the $29$ considered states that is derived by using $\widehat{\rmd}_{\loc}$ in our proposed method in Fig \ref{fig:map} in the Supplementary Material. A brief discussion can also be found therein.



\section{Discussion}
In this work, we consider parameter estimation across multiple devices, under strong constraints such as the disallowance of data sharing, data distribution heterogeneity, limited computational capacity and unstable accessibility of local devices. Under a general $M$-estimation framework, we propose an efficient, scalable, and decentralized algorithm and provide the convergence rate of our estimator in terms of Frobenius norm.  
The clustering consistency, \ie, model selection consistency, of our estimator, as well as statistical properties of the post-clustering estimator can be analyzed in future.  

High-dimensional parameter estimation and inference receive much less attention in federated learning literature. The recent work \citep{battey2018,fan2021,cai2021shir} considered high-dimensional parameter estimation and inference under distributed settings without sharing local data. However, they need to pre-compute a local estimator from each device, requiring a strong computational capacity of local devices. Our method can extend to high-dimensional settings, under all aforementioned constraints in federated learning, by simply adding an extra $\ell_1$-regularization of each parameter, \ie,
\[
\widehat{\mtheta}=\argmin_{\mtheta} \frac{1}{|V|}\sum_{i\in V} \big(\widehat{M}_i(\vtheta_i) + \lambda_i \|\vtheta_i\|_1\big)+ \lambda R(\md\mtheta).
\]
As discussed before, statistical inference in this case still requires further investigation.

\clearpage



\bibliographystyle{jasa}
\bibliography{FedADMM}
\clearpage
\setcounter{section}{0}
\setcounter{equation}{0}
\setcounter{figure}{0}

\renewcommand{\thesection}{S.\arabic{section}}
\renewcommand{\theequation}{S.\arabic{equation}}
\renewcommand{\thelemma}{S.\arabic{lemma}}
\renewcommand{\theproposition}{S.\arabic{proposition}}
\setcounter{figure}{0}
\renewcommand{\thefigure}{S.\arabic{figure}}
\makeatletter
\section*{Supplementary Material for ``Heterogeneous Federated Learning on a Graph''}
This supplementary material contains additional algorithms and figures, and the proofs of theoretical results in Section \ref{sec:method} and Section \ref{sec:optim}.

\section{FedADMM with Heterogeneous Accessibility of Devices}
Following Section \ref{sec:fedadmm_off}, we extend FedADMM to the case where a proportion of devices can be inaccessible during the training process, as illustrated in Algorithm \ref{alg:admm_offline}.
\begin{algorithm}[htp] \SetKwInOut{Input}{Input}\SetKwInOut{Output}{Output} 
\Input{Initial value $\mtheta(0),\bbeta(0),\valpha(0)$, number of iterations $T$, and the learning rate $\eta(t), 1\le t\le T, i\in V$.}
\Repeat{$t> T$}{
\For{$i\in S(t)$}{
Broadcast $\{\vbeta_{i j}(t),\valpha_{i j}(t); j\in N_i\}$ from $\mathcal{S}$ to device $i$
\;
Obtain $\vtheta_i(t + 1)$ by \eqref{eq:update_theta_online} with $p_i=p_i(t)$ in parallel and
send $\vtheta_i(t + 1)$ back to $\mathcal{S}$\;
}
Obtain $\vtheta_i(t+1)$ by \eqref{eq:update_theta_offline} for $i\notin S(t)$ in $\mathcal{S}$\;
Obtain $\vbeta_{i j}(t + 1)$ and $\vbeta_{j i}(t+1)$ by \eqref{eq:update_beta} for $(i,j)\in E$ in $\mathcal{S}$\;
Obtain $\valpha_{i j}(t + 1)$ and $\valpha_{j i}(t+1)$ by \eqref{eq:update_alpha} for $(i,j)\in E$ in $\mathcal{S}$\;
For unknown $\{p_i, i\in V\}$, record $(R_i(t): i\in V)$ and update $p_i(t+1) = t^{-1}\sum_{t=1}^{t}R_i(t)$, $i\in V$\;
$t\leftarrow t+1$\;
}
\Output{$\overline{\mtheta}=T^{-1}\sum_{t=1}^{T}\mtheta(t-1)$}
\caption{Fed-ADMM with randomly inaccessible devices}
\label{alg:admm_offline} 
\end{algorithm}

We remark that the central machine $\mathcal{S}$ is optional, since each device can create a copy of parameters of other devices. Specifically, in the $t$-th iteration, if device $i$ cannot receive $\vtheta_j(t+1)$ from its neighbor (e.g., device $j$), then device $i$ can obtain $\vtheta_j(t+1)$ by \eqref{eq:update_theta_offline}, provided that device $i$ keeps the record of $\valpha_{j k}(t)$ and $\vbeta_{j k}(t)$ for every $k\in N_j$. Despite that this comes with the price of more communication cost, it is still feasible in practice since device $i$ and device $k$ are both connected with device $j$. Moreover, the assumption that $(p_i: i\in V)$ is known, can be relaxed. 
In the $t$-th iteration, we can estimate the vector $(p_i: i\in V)$ using the proportion of each device being offline over the first $t-1$ iterations since $(p_i: i\in V)$ is independent of $t$. The Fed-ADMM with randomly inaccessible devices is summarized in Algorithm \ref{alg:admm_offline}, where, for simplicity, we keep the central machine.
\begin{figure}[htp]
\centering
\includegraphics[width=\textwidth]{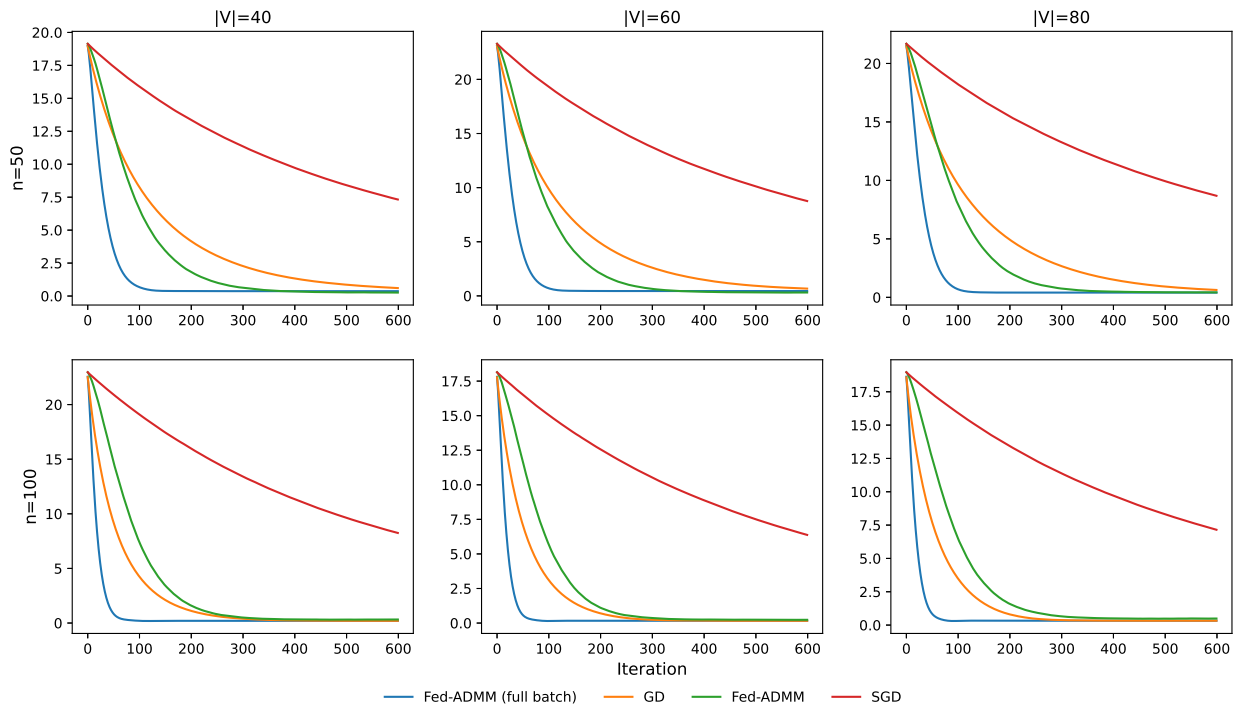}
\caption{Learning curves of the Fed-ADMM with full batch, Fed-ADMM with batch size $10$, GD and SGD with batch size $10$. We set $K=5$, $p=20$, $n=50, 100$ and $|V|=40,60,80$. The $x$-axis and $y$-axis correspond to optimization iteration and average estimation error respectively.
}
\label{fig:S2}
\end{figure}
\section{Additional Figures}
The plot of algorithmic convergence rates of Fed-ADMM as well as its variant,  vanilla gradient descent (GD), and stochastic gradient descent (SGD) is demonstrated in Figure \ref{fig:S2}.

Figure \ref{fig:map} shows the clustering result of using $\widehat{\rmd}_{\loc}$ in our proposed method. The graph contains two connected components, i.e., two clusters. Their members are colored by yellow and blue in Fig \ref{fig:map}, respectively. This suggests that states within each cluster have similar electoral patterns.
There are also some states that are not connected with other states in the graph, which are colored by gray. The statistical associations between predictors and the electoral result of gray states may be different from the majority, and thus data from other states could help little to them for predicting the result. 
\begin{figure}[htp]
\centering
\includegraphics[width=\textwidth]{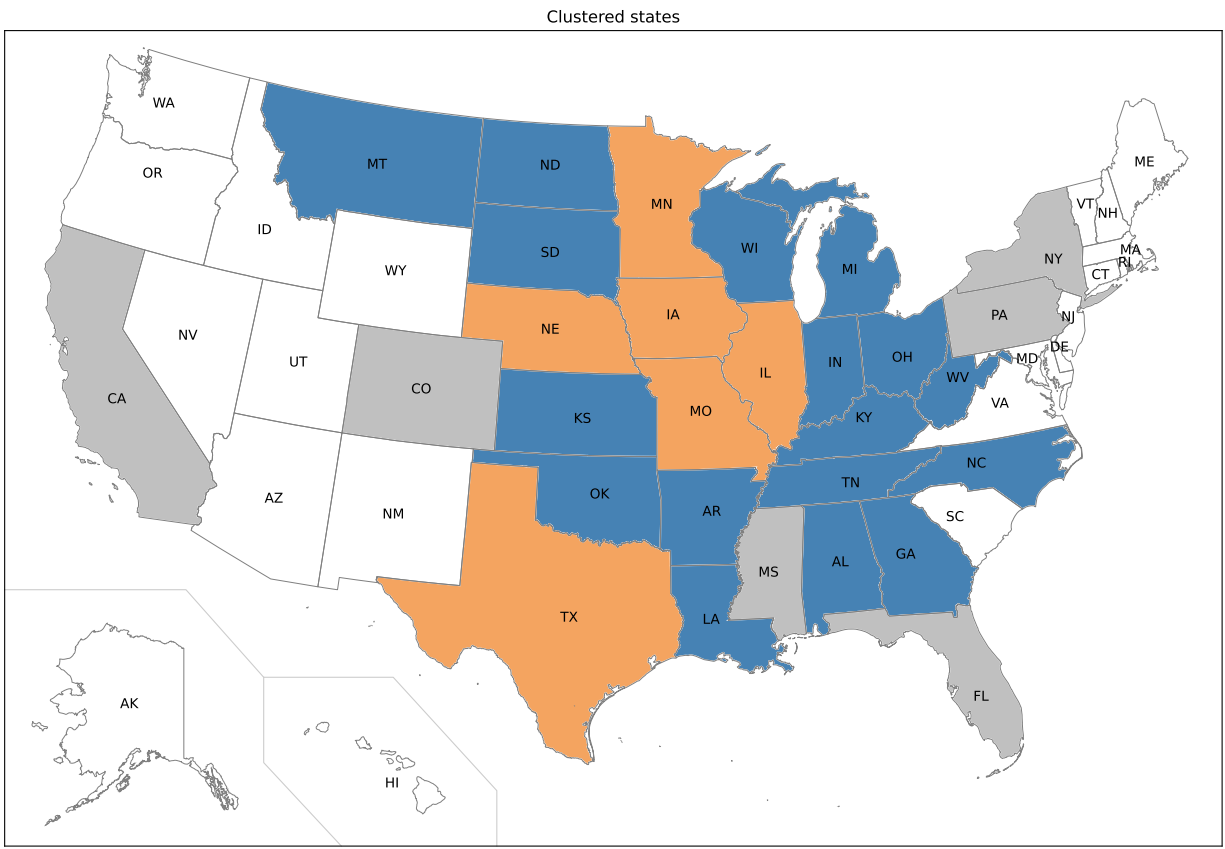}
\caption{Clustering of states under Fed-ADMM with $\widehat{\rmd}_{\loc}$. Yellow and blue states represent two main clusters. Gray indicates states that are not connected to other states. White states are not considered.}
\label{fig:map}
\end{figure}

\section{Proofs for Section \ref{sec:method}}\label{sec:proof}
We first introduce some useful lemmas.
\subsection{Technical Lemmas}\label{sec:el_sg}
We introduce some notation.
For any matrix $\mtheta=(\vtheta_{1},\ldots,\vtheta_{m})^\T \in \mathbb{R}^{m \times p}$, let $\nabla_{\mtheta} R(\mtheta) = (\nabla_{\vtheta} \phi(\vtheta_{1}), \ldots, \nabla_{\vtheta} \phi(\vtheta_{m}))^\T \in \mathbb{R}^{m \times p}$, where $\nabla_{\vtheta}  \phi(\vtheta_{i})$ is the sub-differential of $\phi(\cdot)$ at $\vtheta = \vtheta_{i}$ defined as any vector $\vz$ such that $\phi(\vtheta) - \phi(\vtheta_{i}) \ge \vz^{\T} (\vtheta - \vtheta_{i})$, for all $1\le i\le m$. We define the inner product for two matrices $\matrixu=(\vu_1,\ldots,\vu_r)^\T\in \mathbb{R}^{m\times p}, \mv=(\vv_1,\ldots,\vv_m)^\T \in \mathbb{R}^{m\times p}$ as $\tr( \matrixu\mv^\T)\equiv \sum_{k=1}^m  \vu_k^\T \vv_k$.
The following Lemma characterizes some properties of $R(\cdot)$.
\begin{lemma}[Properties of $R(\cdot)$]\label{lem:Rnorm}
The $\ell_1 / \phi$-norm $R(\cdot):\mathbb{R}^{m \times p}\to [0,\infty)$ admits the following properties:
\begin{itemize}
\item[(\rone)] $R(\cdot)$ is a norm defined on $\mathbb{R}^{m \times p}$.
\item[(\rtwo)] For any $\mtheta_1,\mtheta_2\in\mathbb{R}^{m \times p}$,
\begin{equation}\label{eq:subd_R}
R(\mtheta_1) - R(\mtheta_2) \ge \tr\{ \nabla_{\mtheta} R(\mtheta_2)( \mtheta_1 - \mtheta_2)^\T \}.  
\end{equation}
\item[(\rthree)] Define the dual norm of $\phi(\cdot)$ as $\phi^*(\vv)=\sup_{\phi(\vu) =1}\vu^\T \vv $. Then the dual norm of $R(\cdot)$ is
\begin{equation}\label{eq:dual_R}
R^*(\mv)\equiv\sup_{R(\matrixu) =1}\tr( \matrixu \mv^\T ) = \max_{1\le i\le m}\phi^*(\vv_i).
\end{equation}

\item[(\rfour)] For any $\vu_i\in\mathbb{R}^p$, we have $\vu_i^\T \nabla \phi(\vu_i)  = \phi(\vu_i)$, and  $\tr ( \nabla R(\matrixu) \matrixu^\T ) = R(\matrixu)$.
\item[(\rfive)] $R^*(\nabla R(\mtheta))\le 1$ for any $\mtheta\in \mathbb{R}^{m \times p}$.
\end{itemize}
\end{lemma}
\begin{proof}[Proof of Lemma \ref{lem:Rnorm}] 
Part (\rone) can be directly verified. 

Part (\rtwo) can be proved by noting that
\[
\sum_{i=1}^m  (\vtheta_{1,i} - \vtheta_{2,i})^\T \nabla\phi(\vtheta_{2,i}) \le \sum_{i=1}^m\big(\phi(\vtheta_{1,i}) - \phi(\vtheta_{2,i})\big)= R(\mtheta_1) - R(\mtheta_2)
\]
for any $\mtheta_j=(\vtheta_{j,1},\ldots,\vtheta_{j,m})^\T, j\in\{1,2\}$.

Part (\rthree) can be proved by showing 
\[
\sup_{R(\matrixu) =1}\tr ( \matrixu \mv^\T ) \le \max_{1\le i\le m}\phi^*(\vv_i),
\quad \mathrm{and}\quad 
\sup_{R(\matrixu) =1}\tr ( \matrixu \mv^\T ) \ge \max_{1\le i\le m}\phi^*(\vv_i).
\]
For the first inequality, notice that
\begin{align*}
\sup_{R(\matrixu) =1}\tr ( \matrixu \mv^\T )& = \sup_{R(\matrixu) =1}\sum_{i=1}^m  \vu_i^\T \vv_i  \\
& \le \sup_{R(\matrixu) =1} \sum_{i=1}^m \sup_{\vu_i}\phi(\vu_i) ( \vu_i/\phi(\vu_i) )^\T  \vv_i  \\
& = \sup_{R(\matrixu) =1} \sum_{i=1}^m \phi(\vu_i)\phi^*(\vv_i)\le \max_{1\le i\le m}\phi^*(\vv_i).
\end{align*}
For the second inequality, by choosing $\matrixu_1=(\vu_1,0,0,\ldots,0)^\T$, $\matrixu_2=(0,\vu_2,0,0,\ldots,0)^\T$, $\cdots,\matrixu_m=(0,0,0,\ldots,0,\vu_m)^\T$, we obtain
\begin{align*}
\max_{j\in V}\phi^*(\vv_j) = \max_{j\in V}\sup_{\phi(\vu_j)=1}\tr ( \matrixu_j \mv^\T )   \le \sup_{R(\matrixu) =1} \tr ( \matrixu \mv^\T ).
\end{align*}

Part (\rfour) can be proved by the univariate function $g(\lambda) = \phi(\lambda \vtheta) = \lambda \phi(\vtheta)$ for $\lambda > 0$. Note that $\left. g^\prime(\lambda)\right|_{\lambda = 1}=  \vtheta^\T \nabla \phi(\vtheta)  = \phi(\vtheta) $ by the chain's rule. Thus, 
\[
\tr ( \nabla R(\matrixu) \matrixu^\T ) =\sum_{i=1}^m \vu_i^\T \nabla \phi(\vu_i)  = \sum_{i=1}^m \phi(\vu_i)= R(\matrixu)
\]
for $\matrixu = (\vu_1, \ldots, \vu_m)^\T$.

Part (\rfive) can be proved by noting that
\begin{align}
R^*(\nabla R(\mtheta))& = \sup_{\matrixu\neq\0}\frac{\tr( \nabla R(\mtheta) \matrixu^\T ) }{R(\matrixu)}\notag \\
& = \sup_{\matrixu\neq\0}\frac{\tr\{ \nabla R(\mtheta) (\matrixu - \mtheta )^\T \} + \tr( \nabla R(\mtheta)  \mtheta^\T )  }{R(\matrixu)}\notag\\
&\le \sup_{\matrixu\neq\0}\frac{R(\matrixu) - R(\mtheta)+ R(\mtheta)}{R(\matrixu)}=1, \label{eq:partrtwo}
\end{align}
where \eqref{eq:partrtwo} is due to part (\rtwo) of Lemma \ref{lem:Rnorm}.
\end{proof}

Next, we show that, for each $u\in V$, $M_u(\cdot)$ is both strongly convex and smooth around $\vtheta_u^*$. Let $r_0=\sup_{\va,\vb\in\sxi}\|\va-\vb\|_2$ be the diameter of the parameter space $\sxi$.

\begin{lemma}\label{lem:strong_c}
Suppose that $\sxi \supset \cup_{u\in V}\ball(\vtheta_u^*;\underline{\lambda}/(2L))$.
Under Condition \ref{con:identifiability}, for each $u\in V$,
$M_u(\vtheta)=E[m_u(\rvz;\vtheta)]$ with $\rvz\sim P(\vtheta_u^*)$ is strongly convex and smooth at the vicinity of $\vtheta_u^*$, \ie, $\forall \va\in\ball(\vtheta_u^*;\underline{\lambda}/(2L))$,
\begin{align*}
&M_u(\va)- M_u(\vtheta_u^*)\ge \underline{\lambda}\|\va - \vtheta_u^*\|_2^2/4,\\
&M_u(\va)- M_u(\vtheta_u^*)\le 3\overline{\lambda}\|\va - \vtheta_u^*\|_2^2/4.
\end{align*}

\end{lemma}
\begin{proof}[Proof of Lemma \ref{lem:strong_c}]
Under Condition \ref{con:identifiability}, for any $\va\in \sxi$, applying Taylor's expansion gives, 
\begin{align}
M_u(\va)&=  M_u(\vtheta_u^*) + (\va-\vtheta_u^*)^\T \frac{\partial M_u(\vtheta_u^* )}{\partial \vtheta}\notag \\
& + \frac{1}{2} (\va - \vtheta_u^*)^\T \rmh_u(\vtheta_u^* +\eta_u(\va - \vtheta_u^*) )(\va - \vtheta_u^*),\label{eq:taylor_Mk}
\end{align}
for some $\eta_u\in [0,1]$.
By the Lipschitz continuity of $\rmh_u(\vtheta)$ and Wely's theorem,
\begin{align}
\max_{1\le k\le p}&\big|\lambda_k(\rmh_u(\vtheta_u^* +\eta_u(\va - \vtheta_u^*) )) - \lambda_k(\rmh_u(\vtheta_u^* )) \big|\notag \\
&\le \|\rmh_u(\vtheta_u^* +\eta_u(\va - \vtheta_u^*) ) - \rmh_u(\vtheta_u^* )  \|_2 \le L \|\va - \vtheta_u^*\|_2,\label{eq:perb_H_eta}
\end{align}
where $\lambda_k(\mA)$ denotes the $k$-th largest eigenvalue of the symmetric matrix $\mA$.
By assumption, $ \underline{\lambda}\le \lambda_{\min}(\rmh_u(\vtheta_u^*))\le \lambda_{\max}(\rmh_u(\vtheta_u^*))\le \overline{\lambda}$. Thus, \eqref{eq:perb_H_eta} gives
\[
\inf_{\va\in \ball(\vtheta_u^*;\underline{\lambda}/(2L))}\lambda_{\min}(\rmh_u(\va))\ge \underline{\lambda} - \sup_{\va\in \ball(\vtheta_u^*;\underline{\lambda}/(2L))}L \|\va - \vtheta_u^*\|_2 \ge \underline{\lambda}/2,
\]
and 
\[
\sup_{\va\in \ball(\vtheta_u^*;\underline{\lambda}/(2L))}\lambda_{\max}(\rmh_u(\va))\le \overline{\lambda} + \sup_{\va\in \ball(\vtheta_u^*;\underline{\lambda}/(2L))}L \|\va - \vtheta_u^*\|_2 \le \frac{3\overline{\lambda}}{2}.
\]
Noting that $\partial M_u(\vtheta_u^*) )/\partial \vtheta = 0$, from \eqref{eq:taylor_Mk}, it holds that
\begin{align*}
&M_u(\va)- M_u(\vtheta_u^*)\ge \underline{\lambda}\|\va - \vtheta_u^*\|_2^2/4,\\
&M_u(\va)- M_u(\vtheta_u^*)\le 3\overline{\lambda}\|\va - \vtheta_u^*\|_2^2/4
\end{align*}
for all $\va\in\ball(\vtheta_i^*;\underline{\lambda}/(2L))\subset \sxi$.
\end{proof}

The following lemma shows that empirical Hessian matrices concentrate uniformly to their population counterpart over a neighborhood of true parameters.

\begin{lemma}[Maximal inequality of empirical Hessian matrices]\label{lem:conc_emp_hes}
Under Condition \ref{con:identifiability} and part (\rthree) of Condition \ref{con:distribution}, it holds for all $t\le \sigma_{1,z}^2(\overline{h}-\underline{h})/2$ that
\[
P\left(\max_{u\in V}\sup_{\vtheta\in \sxi}\left\| \widehat{\rmh}_u(\vtheta)- \rmh_u(\vtheta) \right\|_2\ge t\right)\le 2p \sum_{i\in V}  \exp\left(-\frac{n_u t^2}{2\overline{h}^2\sigma_{2,z}^2}  \right).
\]
In particular, with probability at least $1-O\big(p\sum_{i\in V}\exp\{-c_0^2n_u/(18 \overline{h}^2\sigma_{2,z}^2)\}\big)$, we have
\begin{align*}
\max_{i\in V}\sup_{\vtheta\in \mathcal{B}(\vtheta^*_u;c_0/L)}\lambda_{\max}(\widehat{\rmh}_u(\vtheta))&\le \frac{5\overline{\lambda}}{3},\\
\min_{i\in V}\inf_{\vtheta\in \mathcal{B}(\vtheta^*_u;c_0/L)} \lambda_{\min}(\widehat{\rmh}_u(\vtheta_u))&\ge  \underline{\lambda}/3.
\end{align*}
for $c_0 = \min\{\sigma_{1,z}^2(\overline{h}-\underline{h})/2, \overline{\lambda},\underline{\lambda} \} / 3$.
\end{lemma}
\begin{proof}[Proof of Lemma \ref{lem:conc_emp_hes}]
Let $\rmq_{u,k}(\vtheta) = \partial^2 m(\rvz_k;\vtheta)/(\partial  \vtheta\partial \vtheta^\T) - E(\partial^2 m(\rvz_k;\vtheta)/(\partial  \vtheta\partial \vtheta^\T))$ with $\rvz_k\sim P(\vtheta_u^*)$ for some $u\in V$. Hereafter, without loss of generality, we omit the subscript $u$ which denotes the index of devices and write $\rmq_{u,k}(\vtheta)$ as $\rmq_{k}(\vtheta)$. By Condition \ref{con:identifiability}, $m(\cdot;\vtheta)$ is twice continuously differentiable, and $E(\partial^2 m(\rvz_k;\vtheta)/(\partial  \vtheta\partial \vtheta^\T)) = \rmh(\vtheta)$.
Under part (\rthree) of Condition \ref{con:distribution},
\begin{align*}
\rmq_k(\vtheta) & = \left\{h(\rvz_k;\vtheta)\vf(\rvz_k)\vf(\rvz_k)^\T - E[h(\rvz_k;\vtheta)\vf(\rvz_k)\vf(\rvz_k)^\T] \right\}\\
& \preceq \overline{h}\left(\vf(\rvz_k)\vf(\rvz_k)^\T-E[\vf(\rvz_k)\vf(\rvz_k)^\T]\right) + (\overline{h}-\underline{h})E[\vf(\rvz_k)\vf(\rvz_k)^\T]
\end{align*}
uniformly over $\vtheta\in\sxi$.
Let $\msigma_f=E[\vf(\rvz_k)\vf(\rvz_k)^\T]$, and $\rmc_k(\vf)=\vf(\rvz_k)\vf(\rvz_k)^\T-\msigma_f$. By the trace inequality and \citet[Lemma 6.13]{wainwright2019high},
\begin{align}
\sup_{\vtheta\in\sxi} \tr&\left\{E \exp\biggl(\lambda \sum_{k=1}^n \rmq_k(\vtheta)\biggr)\right\} \le  \tr\left\{E \exp\biggl(\lambda \sum_{k=1}^n \overline{h}\rmc_k(\vf) + \lambda n(\overline{h}-\underline{h})\msigma_f\biggr)\right\}\\
& \quad\quad \le \tr\left\{ \exp\left(\sum_{k=1}^n \log E\left\{ \exp\left[\lambda \left( \overline{h}\rmc_k(\vf) +(\overline{h}-\underline{h})\msigma_f\right)\right] \right\} \right)\right\}.\label{eq:tr_exp_log}
\end{align}
By assumption, $E\{\exp[\lambda \overline{h} \rmc_k(\vf)]\}\preceq \exp(\lambda^2\overline{h}^2 \mv / 2)$. Since logarithm is matrix monotone, 
\begin{align*}
\log E &\left\{ \exp\left[\lambda \left( \overline{h}\rmc_k(\vf) +(\overline{h}-\underline{h})\msigma_f\right)\right] \right\} \preceq \log \left[ \exp\left(\frac{\lambda^2\overline{h}^2 \mv}{2} +\lambda (\overline{h}-\underline{h})\msigma_f\right)\right]\\
& = \frac{\lambda^2\overline{h}^2 \mv}{2} +\lambda (\overline{h}-\underline{h})\msigma_f.
\end{align*}
Noting that $\tr(e^{\mv})\le p e^{\|\mv\|_2}$, \eqref{eq:tr_exp_log} can be further bounded as
\begin{align}
\sup_{\vtheta\in\sxi} \tr&\left\{E \exp\biggl(\lambda \sum_{k=1}^n \rmq_k(\vtheta)\biggr)\right\} \le \tr\left\{ \exp\left(\frac{n\lambda^2\overline{h}^2 \mv}{2} +n\lambda (\overline{h}-\underline{h})\msigma_f \right)\right\}\\
&\quad \quad\quad \le p \exp\left\{n\left(\lambda^2\overline{h}^2\left\|\mv\right\|_2/2+\lambda(\overline{h}-\underline{h})\left\|\msigma_f\right\|_2\right)\right\}\\
&\quad \quad\quad \le p \exp\left\{n\left(\lambda^2\overline{h}^2\sigma_{2,z}^2/2+\lambda(\overline{h}-\underline{h})\sigma_{1,z}^2\right)\right\}
\end{align}
By the matrix Chernoff technique, e.g., see \citet[Lemma 6.12]{wainwright2019high},
\begin{align*}
& P\biggl(\sup_{\vtheta\in\sxi}\Big\|\sum_{k=1}^n\rmq_k(\vtheta)\Big\|_2\ge n t
\biggr)\le P\biggl( \sup_{\vtheta\in\sxi}\tr\biggl\{\exp\biggl(\lambda\sum_{k=1}^n \rmq_k(\vtheta)\biggr)\biggr\} \ge e^{\lambda n t}\biggr)\\
&\quad\quad\quad \le 2\sup_{\vtheta\in\sxi} \tr\biggl\{E \exp\biggl(\lambda \sum_{k=1}^n \rmq_u(\vtheta)\biggr)\biggr\}e^{-\lambda n t}\\
&\quad\quad\quad \le 2p \exp\left\{n\left(\lambda^2\overline{h}^2\sigma_{2,z}^2/2+\lambda(\overline{h}-\underline{h})\sigma_{1,z}^2\right)-\lambda n t\right\}.
\end{align*}
Choosing $\lambda = (t-\sigma_{1,z}^2(\overline{h}-\underline{h}))/(\overline{h}^2\sigma_{2,z}^2 )$ gives
\begin{align}\label{eq:conc_Q}
P\left(\sup_{\vtheta\in\sxi}\left\|\sum_{k=1}^n\rmq_k(\vtheta)\right\|_2\ge n t
\right) \le 2p \exp\left(-\frac{n t^2}{2\overline{h}^2\sigma_{2,z}^2}  \right),
\end{align}
for $t\le \sigma_{1,z}^2(\overline{h}-\underline{h})/2$.

Recall that $n_u^{-1}\sum_{k=1}^{n_u}\rmq_k(\vtheta) = \widehat{\rmh}_u(\vtheta) - \rmh_u(\vtheta)$. Applying \eqref{eq:conc_Q} under the choice of $t=\min\{\sigma_{1,z}^2(\overline{h}-\underline{h})/2, \overline{\lambda},\underline{\lambda} \} / 3 \equiv c_0$, and by Weyl's inequality, it holds with probability at least $1-O(p\sum_{u\in V}\exp\{-c_0^2n_u/(18\overline{h}^2\sigma_{2,z}^2)\})$ that
\begin{align*}
\max_{1\le k\le p} \sup_{\vtheta\in\sxi}|\lambda_{k}(\widehat{\rmh}_u(\vtheta)) - \lambda_k(\rmh_u(\vtheta))|&\le \sup_{\vtheta\in\sxi}\|\widehat{\rmh}_u(\vtheta) - \rmh_u(\vtheta)\|_2< c_0,
\end{align*}
where $\lambda_k(\mA)$ denotes the $k$-th largest eigenvalue of the symmetric matrix $\mA$.
Moreover, by Condition \ref{con:identifiability}, $\rmh_u(\vtheta)$ is Lipschitz continuous at $\vtheta_u^*$. Again, by Wely's theorem, we have
\begin{equation*}
\max_{1\le k\le p}|\lambda_k(\rmh_u(\vtheta_u) )) - \lambda_k(\rmh_u(\vtheta_u^* )) |\le \|\rmh_u(\vtheta_u) - \rmh_u(\vtheta_u^* )  \|_2 \le L \|\vtheta_u - \vtheta_u^*\|_2,
\end{equation*}
for all $\vtheta\in\sxi$.
Since $c_0\le \underline{\lambda}/3$, we have $\underline{\lambda}/(2L)\ge c_0/L$ and $\sxi\supset \cup_{u\in V}\ball(\vtheta_u^*;c_0/L)$ hold. We conclude that
\begin{subequations}
\begin{align}
\max_{u\in V}\sup_{\vtheta\in \mathcal{B}(\vtheta^*_u;c_0/L)}\lambda_{\max}(\widehat{\rmh}_u(\vtheta))&\le \overline{\lambda} + c_0 + L \frac{c_0}{L} \le \frac{5\overline{\lambda}}{3},\label{eq:em_lam_max}\\
\min_{u\in V}\inf_{\vtheta\in \mathcal{B}(\vtheta^*_u;c_0/L)} \lambda_{\min}(\widehat{\rmh}_u(\vtheta))&\ge \underline{\lambda} - c_0 - L \frac{c_0}{L}\ge \underline{\lambda}/3.\label{eq:em_lam_min}
\end{align}
\end{subequations}

\end{proof}

We denote by  $\widehat{\rmh}_u(\widehat{\vtheta}_u^{\loc})= n_u^{-1}\sum_{i=1}^{n_u}\partial \vpsi_u(\rvz_i^{(u)};\widehat{\vtheta}_u^{\loc}) /\partial \vtheta$ the plug-in estimator of $\rmh_u(\vtheta_u^*)$. The following lemma shows that $\widehat{\vtheta}_u^{\loc}$ is asymptotic normal.
\begin{lemma}[Asymptotic normality of local estimators]\label{lem:naive_consistency} Under Condition \ref{con:identifiability} and part (\rone,\rthree) of Condition \ref{con:distribution}, we have 
\[
\sqrt{n_u}(\widehat{\vtheta}_u^{\loc}-\vtheta^*_u)\rightarrow_{d}N_p\left(0, \rmh_u(\vtheta_u^*)^{-1}\msigma_u(\vtheta_u^*)\rmh_u(\vtheta_u^*)^{-\T}\right)
\]
as $n_u\to\infty$ for each $u \in V$.
\end{lemma}
\begin{proof}[Proof of Lemma \ref{lem:naive_consistency}]
We first show that, for each $u \in V$, $m_u(\rvz;\vtheta)$ and $\|\vpsi_u(\rvz;\vtheta)\|_2$ are uniformly integrable over $\vtheta\in \sxi$.
By Taylor's expansion, for some $\eta_1,\eta_2\in[0,1]$,
\begin{align}
m_u(\rvz;\vtheta)& \le m_u(\rvz;\vtheta_u^*) +\| \vpsi_u(\rvz;\vtheta_u^*)\|_2 \|\vtheta - \vtheta_u^*\|_2 \notag\\ 
& +  \lambda_{\max}(\nabla \vpsi_u(\rvz;\vtheta_u^* + \eta_1(\vtheta - \vtheta_u^*)))  \|\vtheta - \vtheta_u^*\|_2^2; \notag\\
\|\vpsi_u(\rvz;\vtheta)\|_2&\le \| \vpsi_u(\rvz;\vtheta_u^*)\|_2 + \lambda_{\max}\big(\nabla \vpsi_u(\rvz;\vtheta_u^* + \eta_2(\vtheta - \vtheta_u^*))\big)  \|\vtheta - \vtheta_u^*\|_2.\notag
\end{align}
By part (\rone) of Condition \ref{con:distribution}, $\vpsi_u(\rvz;\vtheta_u^*)$ is sub-Gaussian, and thus $\max_{u\in V}\| \vpsi_u(\rvz;\vtheta_u^*)\|_2$ is integrable. By part (\rtwo) of Condition \ref{con:distribution}, $\nabla \vpsi_u(\rvz;\vtheta)$ is decomposable such that
\[
\nabla \vpsi_u(\rvz;\vtheta) \preceq \overline{h} \vf_u(\rvz)\vf_u(\rvz)^\T
\]
uniformly over $\vtheta\in\sxi$,
where $\vf_u(\rvz)\vf_u(\rvz)^\T$ is a sub-exponential matrix with parameter $(\mv,\alpha)$. By \citet[Lemma 6.12]{wainwright2019high}, and noting $\mathrm{tr}(\exp\{\mA\})\le p\exp\{\|\mA\|_2\}$ for any $p$-dimensional symmetric matrix $\mA$,
\begin{align}
P\biggl(\lambda_{\max}(\vf_u(\rvz)\vf_u(\rvz)^\T)\ge t\biggr)&\le \mathrm{tr}\biggl[\exp\biggl(\frac{\lambda^2\mv}{2(1-\alpha |\lambda|)}\biggr)\biggr]e^{-\lambda t} 
\le p \exp\biggl\{\frac{\lambda^2\sigma_{2,z}^2}{2(1-\alpha |\lambda|) }- \lambda t\biggr\}.\notag
\end{align}
Choosing $\lambda = t/(\sigma_{2,z}^2+\alpha t)\in (0,\alpha^{-1})$, and applying a union bound gives 
\[
P\biggl(\max_{i\in V}\biggl\{\lambda_{\max}(\vf_u(\rvz)\vf_u(\rvz)^\T)\biggr\}\ge t\biggr)\le p|V|\exp\biggl\{-\frac{t^2}{2(\sigma_{2,z}^2+\alpha t)}\biggr\}.
\]
Thus $\max_{u \in V}\sup_{\vtheta\in \sxi}\lambda_{\max}(\nabla \vpsi_u(\rvz;\vtheta))$ has the second moment.

We conclude that there exist integrable functions $\widetilde{m}(\cdot)$ and $\widetilde{\vpsi}(\cdot)$ such that 
\begin{align}
&\max_{u\in V}\sup_{\vtheta\in \sxi}|m_u(\cdot;\vtheta)|\le \widetilde{m}(\cdot)\label{eq:uniform_m},\\
&\max_{u\in V}\sup_{\vtheta\in \sxi}\big\|\vpsi_u(\cdot;\vtheta)\big\|_2\le \widetilde{\vpsi}(\cdot)\label{eq:uniform_psi}.
\end{align}
In fact, \eqref{eq:uniform_m} and \eqref{eq:uniform_psi}, together with the assumption that $m_u(\cdot;\vtheta)$ is twice differentiable with respect to $\vtheta$ implies that the set $\{m_u(\cdot,\vtheta),\vtheta\in{\Theta}\}$ is both Glivenko-Cantelli and Donsker, since $\big|m_u(\cdot;\vtheta_1)-m_u(\cdot;\vtheta_2)\big|\le \widetilde{\vpsi}(\cdot) \|\vtheta_1 -\vtheta_2\|_2$ \citep[see][Chapt. 19.2]{van2000asymptotic}. Thus,
\[
\sup_{\vtheta\in\sxi}|\widehat{M}_u(\vtheta_u)-M_u(\vtheta_u)|\rightarrow_{p}0.
\]
Besides, the strong convexity of $M_u(\cdot)$ implies that ${\sup}_{\vtheta:\|\vtheta-\vtheta_u^*\|\geq \varepsilon} M_u(\vtheta)< M_u(\vtheta^*_u)$ for every $0<\varepsilon\le \underline{\lambda}/(2L)$. Then by \citet[Theorem 5.7, 5.21]{van2000asymptotic} we obtain 
\begin{align}
& \max_{u \in V}\|\widehat{\vtheta}^{\loc}_u - \vtheta_u^* \|_2 \to_p 0. \notag\\
& \sqrt{n_u}(\widehat{\vtheta}_u^{\loc}-\vtheta^*_u)\rightarrow_{d}N_p\left(0, \rmh_u(\vtheta_u^*)^{-1}\msigma_u(\vtheta_u^*)\rmh_u(\vtheta_u^*)^{-\T}\right), \notag
\end{align}
as $n_u\to \infty$.
\end{proof}

\subsection{Proof of Main Theorems}
\begin{proof}[Proof of Theorem \ref{thm:main}]
Let $F(\mtheta)=|V|^{-1}\sum_{u\in V} \widehat{M}_u(\vtheta_u)+ \lambda R(\mD\mtheta)$. Under the assumption $\{\widehat{\vtheta}_u\}_{u\in V}\subset \sxi$, the strong convexity of $\widehat{M}_u(\vtheta_u)$ reveals that $\nabla F(\widehat{\mtheta}) = \0$ for $\widehat{\mtheta}=(\widehat{\vtheta}_u\colon u\in V)=\argmin_{\vtheta_u\in \sxi, u\in V} F(\mtheta)$. Taking derivative with respect to $\vtheta_u$, we obtain
\begin{align*}
\nabla_{\vtheta_u} F(\mtheta) &= |V|^{-1} \nabla\widehat{M}_u(\vtheta_u) +\lambda \sum_{e\in E: u=e^+} \nabla \phi(\vtheta_u - \vtheta_{e^-}) - \lambda \sum_{e\in E: u=e^-} \nabla \phi(\vtheta_{e^+} - \vtheta_{u}) \\
& =|V|^{-1} \nabla\widehat{M}_u(\vtheta_u) +\lambda \sum_{e\in E} \nabla \phi(\vtheta_{e^+} - \vtheta_{e^-}) \big(\1\{u = e^+\} - \1\{u = e^-\}\big).
\end{align*}
Thus, for any fixed $\widetilde{\mtheta}=(\widetilde{\vtheta}_u\colon u\in V)\in \mathbb{R}^{|V|\times p}$,
\begin{align*}
\tr (\nabla F(\widehat{\mtheta}) \widetilde{\mtheta}^\T ) & = 0 =
\sum_{u\in V} \widetilde{\vtheta}_u^\T \nabla_{\vtheta_u} F(\widehat{\mtheta})\\
& = |V|^{-1}\sum_{u\in V} \widetilde{\vtheta}_u^\T \nabla\widehat{M}_u(\widehat{\vtheta}_u) + \lambda \sum_{u\in V} \sum_{e\in E} \widetilde{\vtheta}_u^\T \nabla \phi(\vtheta_{e^+} - \vtheta_{e^-}) \big(\1\{u = e^+\} - \1\{u = e^-\}\big) \\
& = |V|^{-1}\sum_{u \in V} \widetilde{\vtheta}_u^\T \nabla\widehat{M}_u(\widehat{\vtheta}_u) + \lambda\ \tr\Big\{ \big[\nabla R(\mD\widehat{\mtheta})\big](\mD\widetilde{\mtheta})^\T \Big\},
\end{align*}
where $\mD\mtheta\equiv \{\vtheta_{e^+}-\vtheta_{e^-}: e\in E\}$, $R(\mD\mtheta) = \sum_{e\in E}\phi(\vtheta_{e^+}-\vtheta_{e^-})$, and  $\nabla R(\mD \mtheta) = \left. \nabla R(\mdelta)\right|_{\mdelta = \mD\mtheta} =(\nabla\phi(\vtheta_{e^+} - \vtheta_{e^-}): e\in E )$.

By part (3) and (5) of Lemma \ref{lem:Rnorm}, $\big|\tr\{ \nabla R(\mD\widehat{\mtheta})(\mD\widetilde{\mtheta})^\T \}\big| \le R(\mD\widetilde{\mtheta}) R^*(\nabla R(\mD\widehat{\mtheta})) \le R(\mD\widetilde{\mtheta})$. 
Then,
\begin{subequations}
\begin{align}
|V|^{-1}\sum_{u\in V}  -\widetilde{\vtheta}_u^\T \nabla\widehat{M}_u(\widehat{\vtheta}_u) &\le \lambda R(\mD\widetilde{\mtheta}),\label{eq:optimality1}\\
|V|^{-1}\sum_{u\in V} -\widehat{\vtheta}^\T_u \nabla\widehat{M}_u(\widehat{\vtheta}_u)  &= \lambda R(\mD\widehat{\mtheta}).\label{eq:optimality2}
\end{align}
\end{subequations}
Subtracting \eqref{eq:optimality2} from \eqref{eq:optimality1} gives
\begin{align}\label{eq:basic_ineq}
|V|^{-1}\sum_{u\in V}  - (\widetilde{\vtheta}_u - \widehat{\vtheta}_u)^\T \nabla\widehat{M}_u(\widehat{\vtheta}_u) &\le \lambda R(\mD\widetilde{\mtheta}) - \lambda R(\mD\widehat{\mtheta}).
\end{align}
Under Condition \ref{con:identifiability}, for each $u\in V$, $m_u(\rvz;\vtheta)$ is twice continuously differentiable within $\sxi$. The Taylor expansion of $\nabla\widehat{M}_u(\cdot)$ around $\vtheta_u^*$  gives 
\[
\nabla\widehat{M}_u(\widehat{\vtheta}_u) = \nabla\widehat{M}_u(\vtheta_u^*) + \widehat{\rmh}_u(\vtheta_u^* + \eta_u(\widehat{\vtheta}_u - \vtheta_u^*))(\widehat{\vtheta}_u - \vtheta_u^*)
\]
for some $\eta_u\in[0,1]$. For simplicity, let $\widehat{\rmh}_u(\eta_u)= \widehat{\rmh}_u(\vtheta_u^* + \eta_u(\widehat{\vtheta}_u - \vtheta_u^*))$. Plugging this expansion into \eqref{eq:basic_ineq}, we obtain 
\begin{align}\label{eq:basic_ineq2}
\frac{1}{|V|}&\sum_{u\in V}  (-\widetilde{\vtheta}_u + \widehat{\vtheta}_u)^\T \widehat{\rmh}_u(\eta_u)(\widehat{\vtheta}_u - \vtheta_u^*)\notag \\
&\le 
\frac{1}{|V|}\sum_{u\in V} (\widetilde{\vtheta}_u-\widehat{\vtheta}_u)^\T \nabla\widehat{M}_u(\vtheta_u^*)     + \lambda R(\mD\widetilde{\mtheta}) - \lambda R(\mD\widehat{\mtheta}).
\end{align}
Denote by $\mpi_{\ker(\mD)}$ the projection matrix of the kernel space of $\mD$. Observe that $\mpi_{\ker(\mD)} = \mD^{\dag} \mD$, where $\mD^{\dag}$ denotes the Moore-Penrose generalized inverse of $\mD$. By decomposing the identity matrix $\mI_{|V|}=\mI_{|V|} - \mpi_{\ker(\mD)} + \mpi_{\ker(\mD)}$, the Cauchy--Schwarz inequality, and part (\rthree) of Lemma \ref{lem:Rnorm}, we obtain
\begin{align*}
\sum_{u\in V}&(\widetilde{\vtheta}_u-\widehat{\vtheta}_u)^\T \nabla\widehat{M}_u(\vtheta_u^*) = \tr\left(\mpsi_n(\mtheta^*)(\widetilde{\mtheta}-\widehat{\mtheta})^\T\right)\\
&= \tr\left\{\mpsi_n(\mtheta^*)\Big[\big(\mpi_{\ker(\mD)}+ \mI_{|V|}-\mpi_{\ker(\mD)}\big)(\widetilde{\mtheta}-\widehat{\mtheta})\Big]^\T \right\}\\
&\le \|\widetilde{\mtheta}-\widehat{\mtheta}\|_F\|\mpi_{\ker(\mD)}\mpsi_n(\mtheta^*)\|_F + R(\mD(\widetilde{\mtheta}-\widehat{\mtheta}))R^*\{(\mD^{\dag})^\T\mpsi_n(\mtheta^*)\}.
\end{align*}
Choosing $\rho = \|\mpi_{\ker(\mD)}\mpsi_n(\mtheta^*)\|_F/\sqrt{|V|}$ and $\lambda = R^*\{(\mD^{\dag})^\T\mpsi_n(\mtheta^*)\}/\sqrt{|V|}$, \eqref{eq:basic_ineq2} turns out to be
\begin{equation}\label{eq:basic_ineq3}  
\begin{aligned}
\frac{1}{|V|}&\sum_{u\in V}  (-\widetilde{\vtheta}_u + \widehat{\vtheta}_u)^\T \widehat{\rmh}_u(\eta_u)(\widehat{\vtheta}_u - \vtheta_u^*) \le\\
&
\frac{\rho}{\sqrt{|V|}}\|\widetilde{\mtheta}-\widehat{\mtheta}\|_F+ \frac{\lambda}{\sqrt{|V|}}\left( R(\mD(\widetilde{\mtheta}-\widehat{\mtheta})) +  R(\mD\widetilde{\mtheta}) -  R(\mD\widehat{\mtheta})\right)\\
&\le \frac{\rho}{\sqrt{|V|}}\|\widetilde{\mtheta}-\widehat{\mtheta}\|_F
+ \frac{2\lambda}{\sqrt{|V|}}\left\{R[(\mD(\widetilde{\mtheta}-\widehat{\mtheta}))_T]+R[(\mD(\widetilde{\mtheta}))_{T^c}]\right\},
\end{aligned}
\end{equation}
where the last line comes from the triangular inequality of norms. Owing to Definition \ref{def:cf}, we have
\[
R[(\mD(\widetilde{\mtheta}-\widehat{\mtheta}))_T] \le \frac{\sqrt{|T|}}{\kappa_T(\mD)}\|\widetilde{\mtheta}-\widehat{\mtheta}\|_F.
\]
Substituting the above inequality into \eqref{eq:basic_ineq3} gives
\begin{align}
\frac{1}{|V|}\sum_{u\in V}  (-\widetilde{\vtheta}_u + \widehat{\vtheta}_u)^\T& \widehat{\rmh}_u(\eta_u)(\widehat{\vtheta}_u - \vtheta_u^*)  \le
\frac{2\lambda}{\sqrt{|V|}} R\big[(\mD(\widetilde{\mtheta}))_{T^c}\big]\notag \\
& + \frac{1}{\sqrt{|V|}}\|\widetilde{\mtheta}-\widehat{\mtheta}\|_F\left(
\rho + \frac{2\sqrt{|T|  }}{\kappa_T(\mD)}\lambda\right). \label{eq:basic_ineq4}
\end{align}
For any positive semi-definite matrix $\mA$ and two vectors $\va$ and $\vb$, we have
\[
2 \va^\T \mA \vb =  \va^\T \mA \va +  \vb^\T \mA \vb  -  ( \va-\vb)^\T \mA(\va-\vb).
\]
Viewing $\va= \widetilde{\vtheta}_u-\widehat{\vtheta}_u, \vb =  \vtheta_u^*-\widehat{\vtheta}_u $ and $\mA = \widehat{\rmh}_u(\eta_u)$, we obtain
\begin{align}
\frac{\lambda_{\min}(\widehat{\rmh}_u(\eta_u))}{|V|}&\|\widetilde{\mtheta}-\widehat{\mtheta}\|_F^2 + \frac{\lambda_{\min}(\widehat{\rmh}_u(\eta_u))}{|V|}\|\mtheta^*-\widehat{\mtheta}\|_F^2\le \frac{\lambda_{\max}(\widehat{\rmh}_u(\eta_u))}{|V|}\|\widetilde{\mtheta}-\mtheta^*\|_F^2\notag \\
&+  \frac{4\lambda}{\sqrt{|V|}} R\big[(\mD(\widetilde{\mtheta}))_{T^c}\big] 
+\frac{2}{\sqrt{|V|}}\|\widetilde{\mtheta}-\widehat{\mtheta}\|_F\left(
\rho + \frac{2\sqrt{|T|}}{\kappa_T(\mD)}\lambda\right),\label{eq:basic_ineq_5}
\end{align}
where $\lambda_{\max}(\widehat{\rmh}_u(\eta_u))$ denotes the largest eigenvalue of $\widehat{\rmh}_u(\eta_u)$, and $\lambda_{\min}(\widehat{\rmh}_u(\eta_u))$ denotes the smallest eigenvalue of $\widehat{\rmh}_u(\eta_u)$. 
By the Cauchy--Schwarz inequality, 
\[
\frac{2}{\sqrt{|V|}}\|\widetilde{\mtheta}-\widehat{\mtheta}\|_F\left(\rho + \frac{2\sqrt{|T|}}{\kappa_T(\mD)}\lambda\right)\le \frac{\lambda_{\min}(\widehat{\rmh}_u(\eta_u))}{|V|}\|\widetilde{\mtheta}-\widehat{\mtheta}\|_F^2 +  \frac{2}{\lambda_{\min}(\widehat{\rmh}_u(\eta_u))}\left(\rho^2 + \frac{4|T|}{\kappa_T^2(\mD)}\lambda^2\right).
\]
Plugging above inequality into \eqref{eq:basic_ineq_5}, and noting that by part (\rtwo) of Condition \ref{con:distribution}, $\kappa^{-1}\le \lambda_{\min}(\widehat{\rmh}_i(\eta_i))\le \lambda_{\max}(\widehat{\rmh}_i(\eta_i))\le \kappa$,
we have
\begin{equation*}
\begin{aligned}
\frac{1}{|V|}\|\widehat{\mtheta} - \mtheta^*\|_F^2&\le \inf_{\mtheta\in\mathbb{R}^{|V|\times p}}\left\{\frac{\kappa^2}{|V|}\|\mtheta-\mtheta^*\|_F^2 +  4\kappa\lambda R\big[(\mD\mtheta)_{T^c}\big] \right\}\\
&\quad \quad +2\kappa^2\left(\rho^2 + \frac{4|T|}{\kappa_T^2(\mD)}\lambda^2\right).
\end{aligned}   
\end{equation*}
Choosing $T=S$ gives the result.
\end{proof}

In the following, we provide a probabilistic result for $\rho^2$ and $\lambda^2$. 
\begin{proof}[Proof of Theorem \ref{thm:lam_rho}]
We first bound $\rho^2$ in probability. Easy to verify that $\ker(\mD)=\mathrm{span}\big\{\1\{\cc_1\},\ldots,\1\{\cc_{K(E)}\}\big\}$, where $\1\{\cc_i\}\in\mathbb{R}^{|V|}$ denotes the indicator vector of $\cc_i$ whose component in $\cc_i$ is $1$ and otherwise is $0$. By the property of $\mpi_{\ker(\mD)}$, we can write $\rho^2$ explicitly as follows:
\begin{equation}
\rho^2 = \frac{1}{|V|}\sum_{k=1}^{K(E)} \frac{1}{|\cc_k|} \Big\|\sum_{u\in\cc_k} \nabla\widehat{M}_u(\vtheta_u^*)\Big\|_2^2.
\end{equation}
By part (2) of Condition \ref{con:identifiability},  $\nabla\widehat{M}_u(\vtheta_u^*)$ is sub-Gaussian distributed with parameter $\sigma^2/n_u$. Moreover, $\nabla\widehat{M}_u(\vtheta_u^*)\ (u\in V)$ are independent from each other, and thus $\sum_{u\in\cc_k} \nabla\widehat{M}_u(\vtheta_u^*)$ is also sub-Gaussian distributed with parameter $C_1\sum_{u\in \cc_k}\sigma^2/n_u$, where $C_1$ is a universal constant \citep[see, e.g., Section 2.7 of][]{vershynin2018high}. As a consequence, centered $\big\|\sum_{u\in\cc_k} \nabla\widehat{M}_u(\vtheta_u^*)\big\|_2^2$ is sub-exponential distributed with parameter $C_2 p\sum_{u\in \cc_k}\sigma^2/n_u$, and centered $\rho^2$ is sub-exponential distributed with parameter $C_3\sigma^2  p K(E)/(|V|n)$,  where $C_2$ and $C_3$ are also universal constants and $n=\min_{u\in V}n_u$. 

It suffices to bound $E(\rho^2)$.
Under part (\rone) of Condition \ref{con:distribution} and by Jensen's inequality,
\[
\exp\{\sigma^{-2}\vv^\T \msigma_u(\vtheta_u^*) \vv\} =  \exp\{E[\sigma^{-2}(\vv^\T \vpsi_u(\rvz;\vtheta_u^*) )^2]\}\le E \exp\{(\vv^\T \vpsi_u(Z;\vtheta_u^*) )^2\sigma^{-2}\}\le 2.
\]
Thus, $\|\msigma_u(\vtheta_u^*)\|_2\le \sigma^2 \log 2$. Note that
\[
E\big(\|\vpsi_u(\rvz^{(u)};\vtheta_u^*)\|_2^2\big)\le E [\tr(\{\vpsi_u(\rvz^{(u)};\vtheta_u^*)\}^\T \vpsi_u(\rvz^{(u)};\vtheta_u^*))] \le p \|\msigma_u(\vtheta_u^*)\|_2.
\]
In light of the independence among $\nabla\widehat{M}_u(\vtheta_u^*), u\in V$ and the fact $E(\nabla\widehat{M}_u(\vtheta_u^*)(\vtheta_u^*)) = 0$, we further obtain
\begin{align*}
E(\rho^2)&= \frac{1}{|V|}\sum_{k=1}^{K(E)} \frac{1}{|\cc_k|} E\Big(\big\|\sum_{u\in\cc_k} \nabla\widehat{M}_u(\vtheta_u^*)\big\|_2^2\Big)\\
& = \frac{1}{|V|}\sum_{k=1}^{K(E)} \frac{1}{|\cc_k|} \sum_{u\in\cc_k} E \Big(\big\| \nabla\widehat{M}_u(\vtheta_u^*)\big\|_2^2\Big)
\\
&= \frac{1}{|V|}\sum_{k=1}^{K(E)} \frac{1}{|\cc_k|} \sum_{u\in\cc_k} \frac{1}{n_u}E(\|\vpsi_u(\rvz;\vtheta_u^*)\|_2^2)\le  C_4 \sigma^2 \frac{p K(E)}{|V| n},
\end{align*}
where $C_4=(\max_{u\in V}n_u \log 2)/(\min_{u\in V}n_u)$. By the sub-exponential tail,  we have
\begin{equation}\label{size:rho}
P\left(\rho^2-E(\rho^2) > t \right)
\leq \exp\left\{-\frac{|V|n t}{C_5 \sigma^2 p K(E)}\right\},
\end{equation}
for $t > C_3\sigma^2  p K(E)/(|V|n)$ and some constant $C_5$ only depending on $C_3$. Choosing $t=\max\{C_3, C_5\} \log(1/\xi)\sigma^2 p K(E)/(|V|n)$,
it holds with probability at least $1-\xi$ that
\[
\rho^2\leq C_4 \sigma^2 \frac{p K(E)}{|V| n} + \max\{C_3,C_5\}\sigma^2\frac{p K(E)\log(1/\xi)}{|V|n}\le C_{\rho}\sigma^2\frac{p K(E)\log(1/\xi)}{n|V|}
\]
for some constant $C_{\rho}$ depending only on $C_3$--$C_5$.

We next bound $\lambda = R^*\{(\mD^{\dag})^\T\mpsi_n(\mtheta^*)\}/\sqrt{K}$. Denote the $j$-th row of $(\mD^{\dag})^\T$ by $\vs_j^\T, 1\le j\le |E|$. In light of Lemma \ref{lem:Rnorm}, we write
\[
R^*\{(\mD^{\dag})^\T\mpsi_n(\mtheta^*)\} = \max_{j=1,\dots,|E|} \phi^*( \{\mpsi_n(\mtheta^*)\}^\T \vs_j).
\]
For each $j$, $ \{\mpsi_n(\mtheta^*)\}^\T \vs_j$ is sub-Gaussian distributed with parameter $\sum_{u \in V} s_{j u}^2 \sigma^2/n_u\leq C_6\sigma^2(\gamma_G^2 n)^{-1}$ for some universal constant $C_6$. If $\phi^*$ is the supreme norm, by the maximal inequality of sub-Gaussian distributions, it holds with probability at least $1-\xi$ that
\begin{equation}\label{size:lam}
\lambda^2 \leq \left(\frac{C_\lambda\sigma^2}{\gamma_G^2}\right)\frac{p\log(|E|/\xi)}{n |V|},
\end{equation}
for some constant $C_{\lambda}$ only depending on $C_6$.


Substituting \eqref{size:rho} and \eqref{size:lam} into \eqref{thm_main} , we obtain that with probability greater than $1-2\xi$,
\begin{equation*}
\frac{1}{|V|}\|\widehat{\mtheta}-\mtheta^*\|_F^2\le 2\kappa^2 \left\{C_\rho\sigma^2\frac{K(E) p \log(1/\xi)}{n |V|}  + \biggl(\frac{4C_\lambda\sigma^2}{\kappa_0\gamma_G^2}\biggr)\frac{p|S|\log(|E|/\xi)}{n |V|}\right\}.
\end{equation*}
\end{proof}

In the sequel, we provide proofs of theorems and propositions in Section \ref{sec:edge_selection}.
\begin{proof}[Proof of Proposition \ref{prop:optimal_graph_mht}]

By definition, $K(E)$ denotes the number of connected components of $G=(V,E)$. For any $E_1\subset E_2$, we have 
$$K(E_2)\le K(E_1)\le K(E_2) + |E_2\setminus E_1|.$$
For any $\widetilde{E}\subset E$, we can decompose $\widetilde{E}$ into two disjoint sets $\widetilde{E}\cap E_0$ and $\widetilde{E}\setminus E_0$. 
Thus, $K(\widetilde{E}\cap E_0)\ge K(E\cap E_0)$, and $K(\widetilde{E})\ge K(\widetilde{E}\cap E_0) - |\widetilde{E}\setminus E_0|$, since $\widetilde{E}\setminus(\widetilde{E}\cap E_0)=\widetilde{E}\setminus E_0$. This implies that
\begin{align*}
K(\widetilde{E}) + |\widetilde{E}\setminus E_0|\ge K(\widetilde{E}\cap E_0)\ge K(E\cap E_0), 
\end{align*}
for all $\widetilde{E}\subset E$. We conclude that $E\cap E_0$ is one of the minimizers of \eqref{eq:op_graph}.

The first inequality of \eqref{eq:upper_bound_c} is due to the optimality of $E\cap E_0$. The second inequality can be proved by noting
\begin{align*}
K(\widetilde{E})&\le K(\widetilde{E}\cap E_0)\\
&\le K(E\cap E_0) + |(E\cap E_0)\setminus (\widetilde{E}\cap E_0)|\\
&\le K(E\cap E_0)  + |(E\cap E_0)\setminus \widetilde{E}|.
\end{align*}

\end{proof}

\begin{proposition}\label{prop:asymptotic_normality}
Under Condition \ref{con:identifiability}, part (\rone,\rthree) of Condition \ref{con:distribution}, and Condition \ref{con:2ndbartlett}, it holds for each $\vtheta_{e^+}^*=\vtheta_{e^-}^*$ that 
\[
\big(\widehat{\vtheta}_{e^+}^{\loc} - \widehat\vtheta_{e^{-}}^{\loc}\big)^{\T}\big(\widehat{\mupsilon}_{e^+} + \widehat{\mupsilon}_{e^-}\big)^{-1}\big(\widehat{\vtheta}_{e^+}^{\loc} - \widehat\vtheta_{e^{-}}^{\loc}\big)\to_d \chi_p^2,
\]
where for all $u\in V$, $\widehat{\mupsilon}_u=\{n_u \widehat{\rmh}_u(\widehat{\vtheta}_u^{\loc})\}^{-1}$ denotes the asymptotic variance of $\widehat{\vtheta}_u^\loc$. 
\end{proposition}

\begin{proof}[Proof of Proposition \ref{prop:asymptotic_normality}]
By Lemma \ref{lem:naive_consistency} and $\rmh_u(\vtheta_u^*)^{-1}\msigma_u(\vtheta_u^*)\rmh_u(\vtheta_u^*)^{-\T}=\rmh_u(\vtheta_u^*)^{-1}$ by assumption
\[
\sqrt{n_u}(\widehat{\vtheta}_u^{\loc}-\vtheta^*_u)\rightarrow_{d}N_p\left(0, \rmh_u(\vtheta_u^*)^{-1}\right).
\]
By Lemma \ref{lem:conc_emp_hes}, we obtain
\[
\|\widehat{\rmh}_u(\widehat{\vtheta}_u^{\loc}) - \rmh_u(\widehat{\vtheta}_u^{\loc})\|_2\to_p 0.
\]
Moreover, by Condition \ref{con:identifiability}, $\rmh_u(\cdot)$ is Lipschitz continuous at $\vtheta_u^*$ in operator norm. Noting $\widehat{\vtheta}_u^{\loc}\to_p \vtheta_u^*$, by the continuous mapping theorem \citep[Theorem 2.3]{van2000asymptotic}, we have
\begin{align}\label{eq:consit_H}
\|{\rmh}_u(\widehat{\vtheta}_u^{\loc}) - \rmh_u(\vtheta_u^*)\|_2\to_p 0.
\end{align}
By Slutsky's theorem, we obtain
\begin{align}\label{eq:anthetau}
\sqrt{n_u}\{{\rmh}_u(\widehat{\vtheta}_u^{\loc})\}^{1/2}(\widehat{\vtheta}_u^{\loc} - \vtheta_u^*) \to_d N_p(0, \mI_p),
\end{align}
as $n_u\to \infty$. Recall that $\widehat{\mupsilon}_u=\{n_u \widehat{\rmh}_u(\widehat{\vtheta}_u^{\loc})\}^{-1}$ for each $u\in V$. For each $e\in E$ such that $\vtheta_{e^+}^*=\vtheta_{e^-}^*$, noting that $\widehat\vtheta_{e^+}$ and $\widehat\vtheta_{e^-}$ are independent, it further holds that
\[
\big(\widehat{\mupsilon}_{e^+} + \widehat{\mupsilon}_{e^-}\big)^{-1/2} \big(\widehat{\vtheta}_{e^+}^{\loc} - \widehat\vtheta_{e^{-}}^{\loc}\big) \to_d N_p(\0,\mI_p),
\]
owing to continuous mapping theorem, and our result follows.
\end{proof}

\begin{proof}[Proof of Theorem \ref{thm:sel_con}]
Consider the event $\{\widehat{E} = E\cap E_0\}$. Let $A = \{|\widehat{\rcw}_e|^2\le \chi_{p}^2(\alpha/|E|)\mathrm{\ for\ all \ }e\in E\cap E_0\}$ and $B = \{|\widehat{\rcw}_e|^2 > \chi_{p}^2(\alpha/|E|)\mathrm{\ for\ all \ }e\in E\setminus E_0\}$. Noting that $\{\widehat{E} = E\cap E_0\} = A\cap B$, by the union bound, we obtain
\begin{align}
P(\widehat{E} = E\cap E_0)& = P(A\cap B) \ge 1 - P(A^c) - P(B^c)\notag\\
&\ge 1 - |E\cap E_0|\max_{e\in E\cap E_0} P\left(|\widehat{\rcw}_e|^2 > \chi_{p}^2(\alpha/|E|)\right)\notag\\
&\quad \quad - |E\setminus E_0|\max_{e\in E\setminus E_0}P\left(|\widehat{\rcw}_e|^2\le \chi_{p}^2(\alpha/|E|)\right). \label{eq:union_bound_sel}
\end{align}
By Proposition \ref{prop:asymptotic_normality}, $\widehat{\rcw}_e^2\to_d \chi^2_p$ for any $e\in E\cap E_0$; then
\begin{align}\label{eq:false_discovery}
\lim_{n\to\infty}|E\cap E_0|P\left(|\widehat{\rcw}_e|^2> \chi_{p}^2(\alpha/|E|)\right) = \frac{\alpha |E\cap E_0|}{|E|}.
\end{align}
For $e\in E\setminus E_0$, by \eqref{eq:anthetau}, it holds that
\begin{align}\label{eq:test_stat}
\big(\widehat{\mupsilon}_{e^+} + \widehat{\mupsilon}_{e^-}\big)^{-1/2}\big(\widehat{\vdelta}_{e}^{\loc} - \vdelta_e^*\big) \to_d N_p(\0,\mI_p), \quad\widehat{\rca}_e^2\to_d \chi^2_p,
\end{align}
where $\widehat\rca_e^2= \big(\widehat{\vdelta}_{e}^{\loc} - \vdelta_e^*\big)^{\T}\big(\widehat{\mupsilon}_{e^+} + \widehat{\mupsilon}_{e^-}\big)^{-1}\big(\widehat{\vdelta}_{e}^{\loc} - \vdelta_e^*\big)^{\T}$, $\widehat{\vdelta}_{e}^{\loc}=\widehat{\vtheta}_{e^+}^{\loc} - \widehat\vtheta_{e^{-}}^{\loc}$ and $\vdelta_e^*=\vtheta_{e^+}^* - \vtheta_{e^-}^*$. Let $\widehat\rcb_e=\big\{(\vdelta_e^*)^{\T}\big(\widehat{\mupsilon}_{e^+} + \widehat{\mupsilon}_{e^-}\big)^{-1}\vdelta_e^*\big\}^{1/2}$.
By \eqref{eq:consit_H}, $n_{e^+}\widehat{\mupsilon}_{e^+}\to_{p} \mupsilon_{e^+}^*, n_{e^-}\widehat{\mupsilon}_{e^-}\to_{p} \mupsilon_{e^+}^*$, which gives
\[
n_e^{-1/2}\widehat\rcb_e=\big\{(\vdelta_e^*)^{\T}\big(n_e\widehat{\mupsilon}_{e^+} + n_e\widehat{\mupsilon}_{e^-}\big)^{-1}\vdelta_e^*\big\}^{1/2}\to_p \mathrm{dist}(\vtheta_{e^+}^*,\vtheta_{e^-}^*),
\]
where $n_e=\min\{n_{e^+},n_{e^-}\}$.
Under $H_{1,e}$, by triangular inequality, $\widehat{\rcw}_e^2\ge \big(\widehat\rcb_e-\widehat\rca_e\big)^2\ge \widehat\rcb_e^2/2 - \widehat\rca_e^2$; then
\begin{align}
&P\left(\widehat{\rcw}_e^2\le \chi_{p}^2(\alpha/|E|)\right)\le  P\left(\widehat{\rca}_e^2\ge \widehat\rcb_e^2/2 - \chi_{p}^2(\alpha/|E|)\right)\notag \\
& \le P\left(\widehat{\rca}_e^2\ge \biggl(\sqrt{n_e}\mathrm{dist}(\vtheta_{e^+}^*,\vtheta_{e^-}^*)\frac{\widehat\rcb_e}{\sqrt{n_e}\mathrm{dist}(\vtheta_{e^+}^*,\vtheta_{e^-}^*)}\biggr)^{2}\frac{1}{2} - \chi^2_p(\alpha/|E|)\right)\notag\\
& \le P\left(\widehat{\rca}_e^2\ge \biggl(2\sqrt{\chi^2_p(\alpha/|E|)}\frac{\widehat\rcb_e}{\sqrt{n_e}\mathrm{dist}(\vtheta_{e^+}^*,\vtheta_{e^-}^*)}\biggr)^{2}\frac{1}{2} - \chi^2_p(\alpha/|E|)\right)\to \frac{\alpha}{|E|}.\label{eq:false_null}
\end{align}
Combining \eqref{eq:union_bound_sel}, \eqref{eq:false_discovery}, and \eqref{eq:false_null}, we then conclude that
\[
\liminf_{n\to\infty} P\left(\widehat{E} = E\cap E_0\right) \ge 1 - \alpha.
\]
\end{proof}

\section{Proofs for Section \ref{sec:optim}}
In the sequel, we may use $\langle \va,\vb\rangle$ to represent $\va^\T \vb$ for two vectors $\va,\vb\in \mathbb{R}^p$. 
For notational simplicity, we write $\widetilde{\rvg}_i(t) = |\cB_i(t)|^{-1}\sum_{b\in \cB_i(t)}\vpsi_{i}(\rvz_b^{(i)};\vtheta_i(t))$ as the unbiased gradient estimator on device $i$ in the $t$-th iteration. Define $\rvd_i(t+1)\equiv \nabla\widehat{M}_i(\vtheta_i(t)) - \widetilde{\rvg}_i(t)$ as the difference between the gradient estimator and the true gradient. Following the proof technique of \citet{ouyang2013stochastic},
we define an auxiliary function $G(\mtheta,\widetilde{\mtheta},\mbeta,\malpha;\widehat\mtheta,\widehat\mbeta,\widehat\malpha)$ as
\begin{align*}
G(\mtheta,\widetilde{\mtheta},\mbeta,\malpha;\widehat\mtheta,\widehat\mbeta,\widehat\malpha) = \sum_{u\in V}&\left(\widehat{M}_u(\vtheta_u) - \widehat{M}_u(\widehat\vtheta_u)\right) + \lambda \sum_{(i,j)\in E}\left[\phi(  \vbeta_{i j} - \vbeta_{j i} )  - \phi(  \widehat\vbeta_{i j}  - \widehat\vbeta_{j i})\right] \\
& + \sum_{(i,j)\in E} \biggl(\langle  - \widehat\valpha_{i j},\widetilde{\vtheta}_i - \vbeta_{i j}\rangle + \langle  - \widehat\valpha_{j i},\widetilde{\vtheta}_j - \vbeta_{j i}\rangle\biggr),
\end{align*}
where $\mtheta=(\vtheta_i; i\in V)$, $\widetilde{\mtheta}=(\widetilde{\vtheta}_i; i\in V)$, $\mbeta=(\vbeta_{i j},\vbeta_{j i}; (i,j)\in E)$, $\malpha=(\valpha_{i j},\valpha_{j i}; (i,j)\in E)$, and $\widehat{\mtheta},\widehat{\mbeta}$ as well as $\widehat{\malpha}$ denotes the global minimizer of \eqref{eq:ob}.
We first introduce several useful lemmas.
\subsection{Technical Lemmas}\label{sec:tl_opt}
Denote by $D_f(\vx,\vy)=f(\vx) - f(\vy) -(\vx - \vy)^\T \nabla f(\vy) $ the Bregman divergence of $f(\cdot)$, where $f(\cdot)$ is some convex differentiable function with gradient being $\nabla f(\cdot)$. 
\begin{lemma}\label{lem:inner_to_norm}
For any $\va,\vb,\vc,\vd\in\mathbb{R}^p$, we have
\begin{align}
& 2(\va-\vb)^\T (\va-\vc) =\|\va-\vb\|_2^2 + \|\va-\vc\|_2^2 - \|\vb-\vc\|_2^2,\label{eq:three_var}\\
& 2( \va-\vb)^\T(\vc-\vd) = \|\va-\vd\|_2^2 - \|\va-\vc\|_2^2 + \|\vb-\vc\|_2^2 - \|\vb-\vd\|_2^2.\label{eq:four_var}
\end{align}
\end{lemma}

\begin{lemma}[\cite{ouyang2013stochastic}]\label{lem:bregman_div_0}
Let $f(\cdot):\mathbb{R}^p\mapsto \mathbb{R}$ be some convex differentiable function with gradient $\nabla f(\cdot)$. For any $\vy\in \mathbb{R}^p$, and
\[ \vx_* =\argmin_{\vx\in\mathbb{R}^p} f(\vx) + \kappa D_f(\vx, \vy),  \]
we have
\begin{align}
\{\nabla f(\vx_*)\}^\T (\vx_* - \vx) \le \kappa\{D_f(\vx,\vy) - D_f(\vx,\vx_*) - D_f(\vx_*, \vy) \}
\end{align}
\end{lemma}

\begin{lemma}\label{lem:bregman_div_theta}
Under Condition \ref{con:identifiability}, and the update rule \eqref{eq:ADMM_primal}, for any $\vtheta_i\in\mathbb{R}^p$, we have
\begin{align}
& (\vtheta_i(t+1)-\vtheta_i)^\T\biggl(\widetilde{\rvg}_i(t) + \rho\sum_{j\in N_i}\big(\vtheta_i(t+1)-\vbeta_{i j}(t)-\rho^{-1}\valpha_{i j}(t)\big)\biggr)  \notag\\
&\le \frac{1}{2\eta(t+1)}\biggl\{ \|\vtheta_i - \vtheta_i(t)\|_2^2 - \|\vtheta_i - \vtheta_i(t+1)\|_2^2 - \|\vtheta_i(t+1) - \vtheta_i(t)\|_2^2 \biggr\} \label{eq:breg_theta}
\end{align}
\end{lemma}
\begin{proof}[Proof of Lemma \ref{lem:bregman_div_theta}]
Define 
\[
f(\vtheta)=\vtheta^\T\widetilde{\rvg}_i(t) + \frac{\rho}{2}\sum_{j\in N_i}\|\vtheta-\vbeta_{i j}(t) - \rho^{-1}\valpha_{i j}(t)\|_2^2.
\]
Since $f(\cdot)$ is of a quadratic form, it is easy to check that 
$$
D_f(\vtheta_1,\vtheta_2) = \frac{\rho|N_i|}{2}\|\vtheta_1-\vtheta_2\|_2^2.
$$
By Lemma \ref{lem:bregman_div_0}, in view of \eqref{eq:ADMM_primal}, choosing $\vx_* = \vtheta_i(t+1), \vx = \vtheta_i, \vy = \vtheta_i(t)$, we finish the proof by noting 
\[
\nabla f(\vtheta_i(t+1)) = \widetilde{\rvg}_i(t) + \rho\sum_{j\in N_i}(\vtheta_i(t+1)-\vbeta_{i j}(t)-\rho^{-1}\valpha_{i j}(t)).
\]

\end{proof}

We have the following lemma to connect $G(\mtheta,\widetilde{\mtheta},\mbeta,\valpha;\widehat\mtheta,\widehat\mbeta,\widehat\valpha)$ with $\sum_{u\in V}\big(\widehat{M}_u(\vtheta_u) - \widehat{M}_u(\widehat\vtheta_u)\big) + \lambda \sum_{(i,j)\in E}[\phi(  \vbeta_{i j} - \vbeta_{j i} )  - \phi(  \widehat\vbeta_{i j} - \widehat\vbeta_{j i})]$.

\begin{lemma}\label{lem:vadility_G}
Under the condition $\widehat\vbeta_{i j} = \widehat\vtheta_i, \widehat\vbeta_{j i} = \widehat\vtheta_j$, for all $i,j\in V, (i,j)\in E$, we have that
\begin{align*}
& \sup_{\|\widehat\valpha_{i j}\|_2,\|\widehat\valpha_{j i}\|_2\le \kappa_\alpha} G(\mtheta,\widetilde{\mtheta},\mbeta,\valpha;\widehat\mtheta,\widehat\mbeta,\widehat\valpha)  =  \sum_{i\in V}\left(\widehat{M}_i(\vtheta_i) - \widehat{M}_i(\widehat\vtheta_i )\right) \\
&+ \lambda \sum_{(i,j)\in E}\left[\phi(  \vbeta_{i j} - \vbeta_{j i} )  - \phi(  \widehat\vbeta_{i j} - \widehat\vbeta_{j i})\right] + \kappa_\alpha\sum_{(i,j)\in E}\biggl(\|\widetilde{\vtheta}_i - \vbeta_{i j}\|_2 + \|\widetilde{\vtheta}_j -\vbeta_{j i}\|_2\biggr).
\end{align*}
\end{lemma}

\begin{proof}[Proof of Lemma \ref{lem:vadility_G}]
By assumption, $\widehat\vbeta_{i j} = \widehat\vtheta_i , \widehat\vbeta_{j i} = \widehat\vtheta_j$, for all $i,j\in V, (i,j)\in E$. We get
\begin{align*}
& \sum_{(i,j)\in E} \biggl(\langle -{\valpha}_{i j}, \widetilde{\vtheta}_i -\widehat\vtheta_i  - ({\vbeta}_{i j} - \widehat\vbeta_{i j})\rangle + \langle {\valpha}_{i j} - \widehat\valpha_{i j},\widetilde{\vtheta}_i - {\vbeta}_{i j}\rangle\biggr)\\
& + \sum_{(i,j)\in E} \biggl(\langle -{\valpha}_{j i}, \widetilde{\vtheta}_j -\widehat\vtheta_j  - ({\vbeta}_{j i} - \widehat\vbeta_{j i})\rangle + \langle \valpha_{j i} - \widehat\valpha_{j i},\widetilde{\vtheta}_j - \vbeta_{j i}\rangle\biggr)\\
&= -\sum_{(i,j)\in E}\left(\langle \widehat\valpha_{i j}, \widetilde{\vtheta}_i - {\vbeta}_{i j} \rangle+ \langle \widehat\valpha_{j i}, \widetilde{\vtheta}_j - {\vbeta}_{j i} \rangle \right).
\end{align*}
Our proof completes by noting that
\begin{align*}
& \sup_{\|\widehat\valpha_{i j}\|_2,\|\widehat\valpha_{j i}\|_2\le \kappa_\alpha}\biggl\{-\sum_{(i,j)\in E}\left(\langle \widehat\valpha_{i j}, \widetilde{\vtheta}_i - {\vbeta}_{i j} \rangle+ \langle \widehat\valpha_{j i}, \widetilde{\vtheta}_j - {\vbeta}_{j i} \rangle \right)\biggr\}\\
&= -\sum_{(i,j)\in E}\biggl\{ \inf_{\|\widehat\valpha_{i j}\|_2\le \kappa_\alpha}\langle \widehat\valpha_{i j}, \widetilde{\vtheta}_i - {\vbeta}_{i j} \rangle+ \inf_{\|\widehat\valpha_{j i}\|_2\le \kappa_\alpha}\langle \widehat\valpha_{j i}, \widetilde{\vtheta}_j - {\vbeta}_{j i} \rangle\biggr\}\\
& = \kappa_\alpha\sum_{(i,j)\in E}\biggl(\|\widetilde{\vtheta}_i - \vbeta_{i j}\|_2 + \|\widetilde{\vtheta}_j -\vbeta_{j i}\|_2\biggr).
\end{align*}
\end{proof}
\subsection{Proofs of Theorem \ref{thm:conv_alg_admm} and its Corollary}
\begin{proof}[Proof of Theorem \ref{thm:conv_alg_admm}]
Since the augmented Lagrangian function $L(\mtheta,\mbeta,\malpha)$ is strongly convex, it holds for any $(i,j)\in E$ that $\widehat\vtheta_i=\widehat\vbeta_{i j}$, $\widehat\vtheta_j=\widehat\vbeta_{j i}$, where the triplet $\big\{\widehat\mtheta=(\widehat\vtheta_u\colon u\in V), \widehat\mbeta=(\widehat\vbeta_{i j},\widehat\vbeta_{j i}\colon (i,j)\in E), \widehat\malpha=(\widehat\valpha_{i j},\widehat\valpha_{j i}\colon (i,j)\in E)\big\}$ is the global minimizer  of $L(\mtheta,\mbeta,\malpha)$.
For each $(i,j)\in E$, on some rearranging, 
\begin{align*}
\langle \widehat\valpha_{i j},\vbeta_{i j} - \widetilde\vtheta_i\rangle & = \langle \widehat\valpha_{i j} - \valpha_{i j},\vbeta_{i j} - \widetilde\vtheta_i\rangle + \langle \valpha_{i j} , \vbeta_{i j}-\widetilde\vtheta_i\rangle\\
&=\langle \widehat\valpha_{i j} - \valpha_{i j} ,\widetilde\vtheta_i - \vbeta_{i j} \rangle +  \langle \valpha_{i j} , \vbeta_{i j}-\widehat\vbeta_{i j}\rangle + \langle \valpha_{i j}, \widehat\vbeta_{i j} - \widetilde\vtheta_i\rangle.
\end{align*}
We can rewrite $G(\mtheta,\widetilde{\mtheta},\mbeta,\valpha;\widehat\mtheta,\widehat\mbeta,\widehat\valpha)$ as
\begin{align}
&G(\mtheta,\widetilde{\mtheta},\mbeta,\valpha;\widehat\mtheta,\widehat\mbeta,\widehat\valpha) = \sum_{i\in V}\biggl( \widehat{M}_i(\vtheta_i) -  \widehat{M}_i(\widehat{\vtheta}_i) + \sum_{j\in N_i} \langle -\valpha_{i j}, \widetilde{\vtheta}_i - \widehat{\vtheta}_i\rangle\biggr)  \label{eq:fun_theta}\\
&+  \sum_{(i,j)\in E}\left[\lambda\phi(  \vbeta_{i j} - \vbeta_{j i} )  - \lambda \phi(  \widehat\vbeta_{i j} - \widehat\vbeta_{j i}) +  \langle \valpha_{i j} , \vbeta_{i j} - \widehat\vbeta_{i j}\rangle  + \langle \valpha_{j i} ,\vbeta_{j i} - \widehat\vbeta_{j i}\rangle  \right] \label{eq:fun_beta}\\
&+   \sum_{(i,j)\in E} \left[\langle \valpha_{i j} - \widehat\valpha_{i j},\widetilde{\vtheta}_i - \vbeta_{i j}\rangle  + \langle \valpha_{j i} - \widehat\valpha_{j i},\widetilde{\vtheta}_j - \vbeta_{j i}\rangle\right]. \label{eq:fun_alpha}
\end{align}
For simplicity, we omit $(\widehat\mtheta,\widehat\mbeta,\widehat\malpha)$ in $G(\mtheta,\widetilde{\mtheta},\mbeta,\malpha;\widehat\mtheta,\widehat\mbeta,\widehat\malpha)$.
Let $\overline{\mtheta} = T^{-1}\sum_{t=1}^{T}\mtheta(t-1)$, $(\widetilde{\mtheta},\widetilde{\mbeta},\widetilde\malpha)=T^{-1}\sum_{t=1}^{T}(\mtheta(t),\mbeta(t),\valpha(t))$.
It is easy to check that $G(\mtheta,\widetilde{\mtheta},\mbeta,\malpha)$ is convex with respect to $(\mtheta,\widetilde{\mtheta},\mbeta,\malpha)$. Thus, by Jensen's inequality, we have
\begin{equation}\label{eq:jensen}
\begin{aligned}
&  G(\overline{\mtheta},\widetilde{\mtheta},\widetilde{\mbeta},\widetilde\malpha)\le \frac{1}{T}\sum_{t=0}^{T-1}G(\mtheta(t), \mtheta(t+1),\mbeta(t+1),\malpha(t+1)).
\end{aligned}
\end{equation}
We shall construct a telescope structure to prove the result. 

We first deal with \eqref{eq:fun_theta}. Part (\rtwo) of Condition \ref{con:distribution} implies that $\widehat{M}_i(\cdot)$ is strongly convex with parameter $\kappa^{-1}$. We have
\begin{align}
& \widehat{M}_i(\vtheta_i(t)) - \widehat{M}_i(\widehat{\vtheta}_i) \le \{\nabla\widehat{M}_i(\vtheta_i(t))\}^\T(\vtheta_i(t) - \widehat{\vtheta}_i) - \frac{1}{2\kappa}\|\vtheta_i(t) - \widehat{\vtheta}_i\|_2^2 \notag\\
&= (\vtheta_i(t) - \widehat{\vtheta}_i)^\T\widetilde{\rvg}_i(t) + (\vtheta_i(t) - \widehat{\vtheta}_i)^\T\rvd_i(t+1)- \frac{1}{2\kappa}\|\vtheta_i(t) - \widehat{\vtheta}_i\|_2^2\notag \\
&= (\vtheta_i(t+1) - \widehat{\vtheta}_i)^\T\widetilde{\rvg}_i(t) + (\vtheta_i(t) - \widehat{\vtheta}_i)^\T\rvd_i(t+1)\label{eq:Mn_decom}\\
&+(\vtheta_i(t) - \vtheta_i(t+1))^\T\widetilde{\rvg}_i(t)- \frac{1}{2\kappa}\|\vtheta_i(t) - \widehat{\vtheta}_i\|_2^2.\label{eq:first_theta_t}
\end{align}
By the Cauchy--Schwarz inequality, we can bound \eqref{eq:first_theta_t} as
\begin{align}
&(\vtheta_i(t) - \vtheta_i(t+1))^\T\widetilde{\rvg}_i(t) \notag \\
&\le \frac{\eta(t+1)}{2}\|\widetilde{\rvg}_i(t)\|_2^2 + \frac{1}{2\eta(t+1)}\|\vtheta_i(t) - \vtheta_i(t+1)\|_2^2.
\end{align}
Moreover, by \eqref{eq:update_alpha},
\begin{align}
& -(\vtheta_i(t+1) -\widehat{\vtheta}_i)^\T \valpha_{i j}(t+1) =  ( \vtheta_i(t+1) -\widehat{\vtheta}_i)^\T \big\{-\valpha_{i j}(t)+\rho\vtheta_i(t+1) - \rho \vbeta_{i j}(t+1)\big\} \notag \\
&\quad\quad = (\vtheta_i(t+1) -\widehat{\vtheta}_i)^\T \Big\{\rho \big(\vtheta_i(t+1) - \vbeta_{i j}(t) - \rho^{-1}\valpha_{i j}(t)\big) - \rho \big(\vbeta_{i j}(t+1)-\vbeta_{i j}(t)\big)\Big\}\notag\\
&\quad\quad = \rho\big(\vtheta_i(t+1) -\widehat{\vtheta}_i\big)^\T  \big(\vtheta_i(t+1) - \vbeta_{i j}(t) - \rho^{-1}\valpha_{i j}(t)\big)  \label{eq:grad}\\
&\quad\quad\quad\quad - \rho\big(\langle \vtheta_i(t+1) -\widehat{\vtheta}_i\big)^\T \big(\vbeta_{i j}(t+1)-\vbeta_{i j}(t)\big)  \label{eq:theta_beta1}.
\end{align} 
In fact, by \eqref{eq:four_var} in Lemma \ref{lem:inner_to_norm}, \eqref{eq:theta_beta1} can be bounded as
\begin{align}
&  2  \rho( \widehat{\vtheta}_i- \vtheta_i(t+1))^\T (\vbeta_{i j}(t+1)-\vbeta_{i j}(t))  \notag  \\
&= \rho (\|\vtheta_i(t+1) - \vbeta_{i j}(t+1)\|_2^2 - \|\vtheta_i(t+1) - \vbeta_{i j}(t)\|_2^2 ) \notag \\
& + \rho(\|\widehat\vtheta_i  - \vbeta_{i j}(t)\|_2^2 - \|\widehat\vtheta_i  - \vbeta_{i j}(t+1)\|_2^2)\notag\\
&\le \rho(\|\widehat\vtheta_i  - \vbeta_{i j}(t)\|_2^2 - \|\widehat\vtheta_i  - \vbeta_{i j}(t+1)\|_2^2) + \frac{1}{\rho} \|\valpha_{i j}(t+1) - \valpha_{i j}(t)\|_2^2 \label{eq:theta_beta2},
\end{align}
where the last line comes from \eqref{eq:update_alpha}. In light of \eqref{eq:breg_theta} in Lemma \ref{lem:bregman_div_theta}, it holds that
\begin{align}
& (\vtheta_i(t+1)-\widehat\vtheta_i)^\T\biggl(\widetilde{\rvg}_i(t) + \rho\sum_{j\in N_i}\big(\vtheta_i(t+1)-\vbeta_{i j}(t)-\rho^{-1}\valpha_{i j}(t)\big)\biggr)  \notag\\
&\le \frac{1}{2\eta(t+1)}\biggl\{ \|\widehat\vtheta_i - \vtheta_i(t)\|_2^2 - \|\widehat\vtheta_i - \vtheta_i(t+1)\|_2^2 - \|\vtheta_i(t+1) - \vtheta_i(t)\|_2^2 \biggr\} \label{eq:bbreg_theta}
\end{align}
Summing \eqref{eq:Mn_decom}--\eqref{eq:bbreg_theta} up,  we obtain
\begin{equation}\label{eq:bound_fun_theta}
\begin{aligned}
& \widehat{M}_i(\vtheta_i(t)) -  \widehat{M}_i(\widehat{\vtheta}_i) + \sum_{j\in N_i}\langle \vtheta_i(t+1) -\widehat{\vtheta}_i, -\valpha_{i j}(t+1)\rangle\\
&\le \left(\frac{1}{2\eta(t+1)} - \frac{1}{2\kappa}\right)\|\widehat\vtheta_i  - \vtheta_i(t)\|_2^2 - \frac{1}{2\eta(t+1)}\|\widehat\vtheta_i  - \vtheta_i(t+1)\|_2^2\\
& + \frac{\eta(t+1)}{2 }\|\widetilde{\rvg}_i(t)\|_2^2 +  (\vtheta_i(t) - \vtheta_i)^\T\rvd_i(t+1)\\
& + \sum_{j\in N_i}\bigg\{\frac{\rho}{2}\left(\|\widehat\vtheta_i  - \vbeta_{i j}(t)\|_2^2 - \|\widehat\vtheta_i   - \vbeta_{i j}(t+1)\|_2^2\right)+\frac{1}{2\rho} \|\valpha_{i j}(t+1) - \valpha_{i j}(t)\|_2^2\biggr\} .
\end{aligned}
\end{equation}

We then bound \eqref{eq:fun_beta}.
Let $Q(\vbeta_{i j},\vbeta_{j i})=\lambda\phi(\vbeta_{i j} - \vbeta_{j i}) +(\rho/2)\big(\|\vtheta_i(t+1)-\vbeta_{i j} - \rho^{-1}\valpha_{i j}(t)\|_2^2 +
\|\vtheta_j(t+1)-\vbeta_{j i} - \rho^{-1}\valpha_{j i}(t)\|_2^2\big)$.  Under the update rule \eqref{eq:update_beta}, since $Q(\cdot,\cdot)$ is convex, we obtain
\begin{align}
&\nabla_{\vbeta_{i j}}Q(\vbeta_{i j}(t+1),\vbeta_{j i}(t+1))  = \lambda \partial \phi(\vbeta_{i j}(t+1)-\vbeta_{j i}(t+1)) \notag \\
&\quad\quad\quad + \rho\big(\vbeta_{i j}(t+1) + \rho^{-1} \valpha_{i j}(t) - \vtheta_i(t+1)\big)=0,\label{eq:partial_ij}\\
& \nabla_{\vbeta_{j i}}Q(\vbeta_{i j}(t+1),\vbeta_{j i}(t+1))  = - \lambda\partial \phi(\vbeta_{i j}(t+1)-\vbeta_{j i}(t+1)) \notag  \\
&\quad\quad\quad + \rho\big(\vbeta_{j i}(t+1) + \rho^{-1} \valpha_{j i}(t) - \vtheta_j(t+1)\big) = 0,\label{eq:partial_ji}
\end{align}
where $\partial \phi$ denotes the sub-differential of $\phi$. For any $\vbeta_{i j},\vbeta_{j i}\in\mathbb{R}^p$, it follows
\begin{align}
\langle \widehat\vbeta_{i j},\nabla_{\vbeta_{i j}}Q(\vbeta_{i j}(t+1),\vbeta_{j i}(t+1))\rangle + \langle \widehat\vbeta_{j i},\nabla_{\vbeta_{j i}}Q(\vbeta_{i j}(t+1),\vbeta_{j i}(t+1))\rangle = 0. \label{eq:sub_diff}
\end{align}
Note that for any $\vu,\vv\in\mathbb{R}^p$, $\phi^*(\partial \phi(\vv))\le 1$ and $|\vu^\T\partial \phi(\vv)|\le \phi(\vu)$, with equality holds if and only if $\vv=\vu$. Combining \eqref{eq:partial_ij}, \eqref{eq:partial_ji}, \eqref{eq:sub_diff}, and the update rule \eqref{eq:update_alpha} gives
\begin{align*}
\lambda \phi(\vbeta_{i j}(t+1)-\vbeta_{j i}(t+1)) + \langle \vbeta_{i j}(t+1), \valpha_{i j}(t+1)\rangle + \langle \vbeta_{j i}(t+1), \valpha_{j i}(t+1)\rangle &= 0, \\
- \lambda \phi(\widehat\vbeta_{i j}-\widehat\vbeta_{j i}) - \big(\langle \widehat\vbeta_{i j}, \valpha_{i j}(t+1)\rangle + \langle \widehat\vbeta_{j i}, \valpha_{j i}(t+1)\rangle \big) &\le 0.
\end{align*}
Thus, 
\begin{align}
\sum_{(i,j)\in E}& \biggl\{\lambda \phi(\vbeta_{i j}(t+1)-\vbeta_{j i}(t+1)) - \lambda \phi(\widehat\vbeta_{i j}-\widehat\vbeta_{j i})\biggr. \notag \\
& \biggl.+ \langle \vbeta_{i j}(t+1) - \widehat\vbeta_{i j}, \valpha_{i j}(t+1)\rangle + \langle \vbeta_{j i}(t+1) - \widehat\vbeta_{j i}, \valpha_{j i}(t+1)\rangle\biggr\}\le 0.\label{eq:bound_fun_beta}
\end{align}

We finally bound \eqref{eq:fun_alpha} by \eqref{eq:three_var} in Lemma \ref{lem:inner_to_norm} as follows:  
\begin{align}
& \langle \valpha_{i j}(t+1) - \widehat\valpha_{i j}, \vtheta_i(t+1) - \vbeta_{i j}(t+1)\rangle = \rho^{-1} \langle \valpha_{i j}(t+1) - \widehat\valpha_{i j}, \valpha_{i j}(t) - \valpha_{i j}(t+1)\rangle \notag \\
&  =\frac{1}{2\rho}\biggl\{ \|\widehat\valpha_{i j} - \valpha_{i j}(t)\|_2^2 - \|\widehat\valpha_{i j} - \valpha_{i j}(t+1)\|_2^2 - \|\valpha_{i j}(t+1) - \valpha_{i j}(t)\|_2^2\biggr\}. \label{eq:bound_fun_alpha}
\end{align}

Combining \eqref{eq:bound_fun_theta}, \eqref{eq:bound_fun_beta}, and \eqref{eq:bound_fun_alpha}, we obtain
\begin{equation}\label{eq:basic_tele}
\begin{aligned}
& G(\mtheta(t), \mtheta(t+1),\mbeta(t+1),\malpha(t+1))\\
& \le
\sum_{i\in V}\biggl\{\left(\frac{1}{2\eta(t+1)} - \frac{1}{2\kappa}\right)\|\widehat\vtheta_i  - \vtheta_i(t)\|_2^2 - \frac{1}{2\eta(t+1)}\|\widehat\vtheta_i  - \vtheta_i(t+1)\|_2^2\biggr.\\
&\qquad\qquad  + \frac{\eta(t+1)}{2}\|\widetilde{\rvg}_i(t)\|_2^2 + (\vtheta_i(t) - \widehat{\vtheta}_i)^\T\rvd_i(t+1)\\
& \qquad\qquad + \biggl.\sum_{j\in N_i}\frac{\rho}{2}\left(\|\widehat\vtheta_i  - \vbeta_{i j}(t)\|_2^2 - \|\widehat\vtheta_i  - \vbeta_{i j}(t+1)\|_2^2\right)\biggr\}\\
& \qquad\qquad +\sum_{(i,j)\in E}\biggl\{ \frac{1}{2\rho}\left( \|\widehat\valpha_{j i} - \valpha_{j i}(t)\|_2^2 - \|\widehat\valpha_{j i} - \valpha_{j i}(t+1)\|_2^2\right)\biggr.\\
&\qquad\qquad \biggl.+\frac{1}{2\rho}\left( \|\widehat\valpha_{i j} - \valpha_{i j}(t)\|_2^2 - \|\widehat\valpha_{i j} - \valpha_{i j}(t+1)\|_2^2\right)\biggr\}.
\end{aligned}
\end{equation}
Under the choice of $\eta(t) = \kappa / t$, $(1/\eta(t+1) - 1/\kappa) = t/\kappa$.
In view of \eqref{eq:jensen}, summing \eqref{eq:basic_tele} with $t=0,\ldots,T-1$ gives, 
\begin{align}
& G(\overline{\mtheta},\widetilde{\mtheta},\overline{\mbeta},\overline\malpha;\widehat\malpha)
\le \frac{1}{2T}\sum_{(i,j)\in E}\biggl\{\frac{1}{\rho}\biggl( \|\widehat\valpha_{i j} - \valpha_{i j}(0)\|_2^2 + \|\widehat\valpha_{j i} - \valpha_{j i}(0)\|_2^2\biggr)+ \rho r_0^2 \biggr\}\label{eq:G_without_exp_1} \\
& \qquad\qquad+\frac{1}{T}\sum_{i\in V}\sum_{t=0}^{T-1}\biggl(\frac{\eta(t+1)}{2}\|\widetilde{\rvg}_i(t)\|_2^2 + (\vtheta_i(t) - \widehat{\vtheta}_i)^\T\rvd_i(t+1)\biggr)\label{eq:G_without_exp_2}\\
&\qquad \qquad+ \frac{1}{T}\sum_{i\in V}\biggl\{-\frac{T}{2\kappa}\|\vtheta_i(T)-\widehat{\vtheta}_i\|_2^2+ \sum_{j\in N_i} \frac{\rho}{2}\|\widehat\vtheta_i  - \vbeta_{i j}(0)\|_2^2\biggr\}.
\end{align}
By assumption,
\begin{align*}
& \max_{i\in V, j\in N_i}\|\widehat\vtheta_i - \vbeta_{i j}(0)\|_2^2\le 4r_0^2\\
&  \max_{i\in V, j\in N_i} (\|\valpha_{i j}(0)\|_2^2 + \|\widehat\valpha_{i j}\|_2^2)= \kappa_{\alpha}\\
&\max_{i\in V} E\left(\|\widetilde{\rvg}_i(t)\|_2^2\right)\le \max_{i\in V} \frac{1}{n_i}\sum_{k=1}^{n_i}\|\vpsi_i(\rvz_k^{(i)};\vtheta_i(t))\|_2^2\le C_{\psi}.
\end{align*}
Moreover, since the mini-batch at each iteration is independently sampled, $\rvd_i(t+1)$ is independent of previous updates, and $E\big((\vtheta_i(t) - \widehat{\vtheta}_i)^\T\rvd_i(t+1)\big)\equiv 0$ for each $i\in V$. Thus,
combining \eqref{eq:G_without_exp_1}, \eqref{eq:G_without_exp_2} with Lemma \ref{lem:vadility_G}, and taking expectation with respect to the choice of mini-batches $\{\cB_u(t)\colon u\in V, 1\le t\le T\}$, we have
\begin{align}
& E\biggl\{\sup_{\|\widehat\valpha_{i j}\|_2, \|\widehat\valpha_{j i}\|_2\le \kappa_{\alpha}} G(\overline{\mtheta},\widetilde{\mtheta},\widetilde{\mbeta},\widetilde\valpha;\widehat\mtheta,\widehat\mbeta,\widehat\valpha) \biggr\}= E\biggl\{ \sum_{i\in V}\left(\widehat{M}_i(\overline{\vtheta}_i) - \widehat{M}_i(\widehat{\vtheta}_i)\right)\biggr.\notag\\
&+\biggl. \lambda \sum_{(i,j)\in E}\left[\phi(  \widetilde{\vbeta}_{i j} - \widetilde{\vbeta}_{j i} )  - \phi(  \widehat\vbeta_{i j} - \widehat\vbeta_{j i})\right] + \kappa_{\alpha} \sum_{(i,j)\in E}\|\widetilde{\vtheta}_i - \widetilde{\vbeta}_{i j}\|_2 + \|\widetilde{\vtheta}_j - \widetilde{\vbeta}_{j i}\|_2 \biggr\}\notag\\
&\le \frac{1}{2T}\sum_{(i,j)\in E}\biggl\{\frac{1}{\rho}\sup_{\|\widehat\valpha_{i j}\|_2, \|\widehat\valpha_{j i}\|_2\le \kappa_{\alpha}}\biggl( 2\|\widehat\valpha_{i j}\|_2^2 + 2\|\valpha_{i j}(0)\|_2^2 + 2\|\widehat\valpha_{j i}\|_2^2 + 2\|\valpha_{j i}(0)\|_2^2\biggr) \biggr\}\label{eq:G_exp_1}\\
& + \frac{1}{T}\sum_{i\in V}\sum_{t=0}^{T-1}E\biggl(\frac{\eta(t+1)}{2}\|\widetilde{\rvg}_i(t)\|_2^2 + (\vtheta_i(t) - \widehat\vtheta_i)^\T\rvd_i(t+1)\biggr)+ \frac{2\rho r_0^2 |E|}{T} \label{eq:G_exp_2}\\
&\le \frac{|V|C_{\psi}}{2T}\sum_{t=1}^T\frac{\kappa}{t} + \frac{|E|}{T}\biggl(\frac{8\kappa_{\alpha}}{\rho}+2\rho r_0^2\biggr) \le \frac{\kappa C_{\psi} |V|\log T}{T}+\frac{|E|}{T}\biggl(\frac{8\kappa_{\alpha}}{\rho}+2\rho r_0^2\biggr).\notag
\end{align}

To obtain the result, note that $\widetilde{\vtheta}_i=\overline{\vtheta}_i+T^{-1}(\vtheta_i(T)-\vtheta_i(0))$
and thus by triangular inequality
\[
\|\widetilde{\vtheta}_i - \widetilde{\vbeta}_{i j}\|_2 \ge \|\overline{\vtheta}_i - \widetilde{\vbeta}_{i j}\|_2 -T^{-1}\|\vtheta_i(T)-\vtheta_i(0)\|_2\ge\|\overline{\vtheta}_i - \widetilde{\vbeta}_{i j}\|_2 -\frac{2r_0}{T} 
\]
Moreover, 
by assumption, $\kappa_{\alpha}\ge \lambda \sup_{\va\neq \0}\phi(\va)/\|\va\|_2$, we have $\kappa_{\alpha}\|\overline{\vtheta}_i - \widetilde{\vbeta}_{i j}\|_2 \ge \lambda \phi(\overline{\vtheta}_i - \widetilde{\vbeta}_{i j})$, and similarly, $\kappa_{\alpha}\|\overline{\vtheta}_j - \widetilde{\vbeta}_{j i}\|_2 \ge \lambda \phi(\overline{\vtheta}_j - \widetilde{\vbeta}_{j i})$. Again, by triangular inequality, it holds that
\[
\phi(\overline{\vtheta}_i - \widetilde{\vbeta}_{i j}) + \phi(\overline{\vtheta}_j - \widetilde{\vbeta}_{j i}) + \phi(\widetilde{\vbeta}_{i j} - \widetilde{\vbeta}_{j i}) \ge \phi(\overline{\vtheta}_i - \overline{\vtheta}_j).
\]
Let $F(\mtheta)=\sum_{i\in V}\widehat{M}_i(\vtheta_i)+\lambda \sum_{(i,j)\in E}\phi(\vtheta_i-\vtheta_j)$. Taking expectation with respect to the choice of mini-batches $\{\cB_u(t)\colon u\in V, 1\le t\le T\}$, it holds that
\begin{align*}
E\big[F(\overline\mtheta)\big] - F(\widehat\mtheta) &\le E\biggl\{\sup_{\|\widehat\valpha_{i j}\|_2, \|\widehat\valpha_{j i}\|_2\le \kappa_{\alpha}} G(\overline{\mtheta},\widetilde{\mtheta},\widetilde{\mbeta},\widetilde\valpha;\widehat\mtheta,\widehat\mbeta,\widehat\valpha) \biggr\} + \frac{4r_0\kappa_{\alpha}|E|}{T}\\
&\le \frac{\kappa C_{\psi} |V|\log T}{T}+\frac{|E|}{T}\biggl(\frac{8\kappa_{\alpha}}{\rho}+2\rho r_0^2+ 4r_0\kappa_{\alpha}\biggr)\le \frac{2\kappa C_{\psi} |V|\log T}{T},
\end{align*}
for sufficiently large $T$, \ie, $\kappa C_{\psi} |V|\log T \ge |E|(8\rho^{-1}\kappa_{\alpha} + 2r_0^2 \rho + 4r_0\kappa_\alpha)$. Since $F(\cdot)$ is strongly convex with parameter $|V|\kappa$ under part (\rtwo) of Condition \ref{con:distribution}, we conclude that
\[
\frac{1}{|V|}E\left(\left.\|\overline{\mtheta}-\widehat{\mtheta}\|_F^2\right|\rvz_k^{(u)}, 1\le k\le n_u, u\in V\right)\le \frac{2\kappa^2 C_{\psi} |V|\log T}{T}.
\]

\end{proof}

\begin{proof}[Proof of Corollary \ref{coro:conv_alg_admm}] 
Define $\widetilde{\rvg}_i^{\mathrm{ipw}} (t) = p_i^{-1}\widetilde{\rvg}_i (t) \1\{i\in S(t+1)\} $. 
Since for any fixed $t$ and $i$, the choice of $\cB_i(t)$ and $\1\{i\in S(t)\}$ are independent, we obtain 
\begin{align}
& E(\widetilde{\rvg}_i^{\mathrm{ipw}} (t)) = E(\widetilde{\rvg}_i (t))E(p_i^{-1}\widetilde{\rvg}_i (t)) = \nabla\widehat{M}_i(\vtheta_i(t)).\notag
\end{align}
The arguments of proving Theorem \ref{thm:conv_alg_admm} just carry over due to the unbiasedness of $\widetilde{\rvg}_i^{\mathrm{ipw}} (t)$, except noting that
\[
E\left(\big\|\widetilde{\rvg}_i^{\mathrm{ipw}} (t)\big\|_2^2\right)\le p_0^{-1}C_{\psi}.
\]

\end{proof}

\end{document}